\RequirePackage{fix-cm}
\documentclass[smallcondensed]{svjour3}
\smartqed  

\usepackage{natbib}
\usepackage{enumitem}
\usepackage{mathtools}
\usepackage{color} 
\usepackage{amsmath}
\usepackage{amssymb}
\usepackage{graphicx}
\usepackage{subfigure}
 \usepackage{hyperref}
\graphicspath{{./}}
\newcommand{\edit}[1]{\textcolor{black}{#1}}
\newcommand{\editnew}[1]{\textcolor{black}{#1}}
\DeclarePairedDelimiter{\ceil}{\lceil}{\rceil}

\def\R{\mathbb{R}}

\def\X{\mathcal{X}}
\def\F{\mathcal{F}}
\def\B{\mathcal{B}}
\def\Y{\mathcal{Y}}
\def\matXbar{{X_{sL}}}
 \def\xt{\tilde{x}}
 \def\yt{\tilde{y}}
 \def\scalej{\frac{n^{1/r}X_{b}B_{b}}{\|h_{j}\|_{r}}}

\def\E{\mathbb{E}}
\def\Aint{A_{\textrm{int}\gamma}}
\def\Rad{\mathcal{\bar{R}}(\F_{|S})}
\def\Gauss{\mathcal{\bar{G}}(\F_{|S})}
\def\Xlab{{X}_{L}}
\def\XunlabT{{X}_{U}^{T}}
\def\Pr{\mathbb{P}}
\def\GaussScaled{n\cdot\Gauss}
\renewcommand\cite{\citep}
\allowdisplaybreaks[1] 
\def\AppendixA{Appendix A}

\begin{document}
\title{Generalization Bounds for Learning with Linear, Polygonal, Quadratic and Conic Side Knowledge}
\titlerunning{Learning with Linear, Polygonal, Quadratic and Conic Side Knowledge}        
\author{Theja Tulabandhula \and Cynthia Rudin}
\institute{Theja Tulabandhula \at
               Department of Electrical Engineering and Computer Science,\\ 
               Massachusetts Institute of Technology, Cambridge, MA 02139, USA.\\
               \email{theja@mit.edu}
           \and
              Cynthia Rudin \at
              MIT Sloan School of Management,\\
       	     Massachusetts Institute of Technology, Cambridge, MA 02139, USA.\\
              \email{rudin@mit.edu}
}
\journalname{Mach Learn}
\date{Received: date / Accepted: date}
\maketitle


\begin{abstract}
In this paper, we consider a supervised learning setting where side knowledge is provided about the labels of unlabeled examples. The side knowledge has the effect of reducing the hypothesis space, leading to tighter generalization bounds, and thus possibly better generalization. We consider \edit{several} types of side knowledge, the first leading to linear \edit{and polygonal constraints} on the hypothesis space, the second leading to quadratic constraints,  \edit{and the last leading to conic constraints}. We show how different types of domain knowledge can lead directly to these kinds of side knowledge.  We prove bounds on complexity measures of the hypothesis space for quadratic  \edit{and conic} side knowledge, and show that these bounds are tight in a specific sense  \edit{for the quadratic case}.
\end{abstract}



\section{Introduction} \label{sec:intro}
Surely, for many applications
the amount of domain knowledge we could potentially use within our learning processes is vastly larger than the amount of domain knowledge we actually use. 
One reason for this is that domain knowledge may be nontrivial to incorporate into algorithms or analysis. A few types of domain knowledge that do permit analysis have been explored quite in depth in the past few years and used very successfully in a variety of learning tasks; this includes knowledge about the sparsity properties of linear models ($\ell_{1}$-norm constraints, minimum description length) or smoothness properties ($\ell_{2}$-norm constraints, maximum entropy). 
A reason that domain knowledge is not usually incorporated in theoretical analysis is that it can be very problem specific; it may be too specific to the domain to have an overarching theory of interest. For example, researchers in NLP (Natural Language Processing) have long figured out various exotic domain specific knowledge that one can use while performing a learning task \cite{chang2008constraints,chang2008learning}. 
The present work aims to provide theoretical guarantees for a large class of problems with a general type of domain knowledge that goes beyond sparsity and smoothness.

To define this large class of problems, we will keep the usual supervised learning assumption that the training examples are drawn i.i.d. Additionally in our setting, we have a different set of examples without labels, not necessarily chosen randomly. For this set of unlabeled examples, we have some prior knowledge about the relationships between their labels, which affects the space of hypotheses we are searching over within our learning algorithms.   \edit{We motivate this knowledge as being obtained from domain experts.} These assumptions can, for example, take into account our partial knowledge about how any learned model should predict on the unlabeled examples if they were encountered. We consider \edit{many} types of side knowledge, namely constraints on the unlabeled examples leading to (i) linear constraints on a linear function class, (ii) quadratic constraints on a linear function class, and (iii) conic constraints on a linear function class. Our main contributions are:
\begin{itemize}[noitemsep,topsep=0pt,parsep=0pt,partopsep=0pt,leftmargin=*]
\item To show that linear, \edit{polygonal}, quadratic \edit{and conic} constraints on a linear hypothesis space can arise naturally in many circumstances, from constraints on a set of unlabeled examples. This is in Section \ref{sec:structure}. We connect these with relevant semi-supervised learning settings.
\item \edit{To provide upper bounds on covering number and empirical Rademacher complexity for linearly constrained linear function classes.} Bounds for the case of linear \edit{and polygonal} constraints are found in Sections \ref{subsec:linear-bdd} \edit{and \ref{subsec:multiple-linear-bdd} respectively. Two of the three} bounds in these sections are not original to this paper, but their application to general side knowledge with linear constraints is novel.
\item To provide \edit{two upper} bounds on the complexity of the hypothesis space for the quadratic constraint case
This can be used directly in generalization bounds. The \edit{use of a certain family of circumscribing ellipsoids and the} quadratic bounds of Section \ref{subsec:quadratic-bdd} are novel to this paper.
\item To show that \edit{one of the} upper bound\edit{s} on the quadratically constrained hypothesis space we provided has a matching lower bound, also in Section \ref{subsec:quadratic-bdd}. \edit{This is novel to this paper}.
\item\edit{To provide a bound on the complexity of the hypothesis space for the conic constraint case. These bounds are in Section \ref{subsec:conic-bdd} and are novel to this paper.}
\item\edit{We develop a novel proof technique for upper bounding linear, quadratic and conic constraint cases based on convex duality.}
\end{itemize}
Figure \ref{fig:side-knowledge} illustrates the various types of side knowledge.

\begin{figure}
     \centering
     \subfigure[]{
     \includegraphics[width=.22\textwidth]{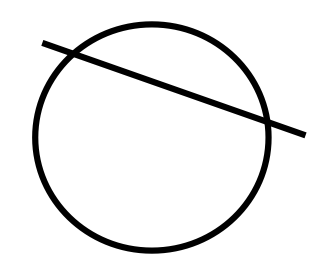}
     \label{fig:linear}
     }
     \subfigure[]{
     \includegraphics[width=.22\textwidth]{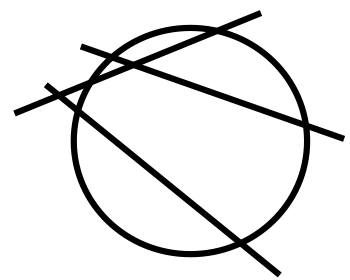}
     \label{fig:polygon}
     }
     \subfigure[]{
     \includegraphics[width=.21\textwidth]{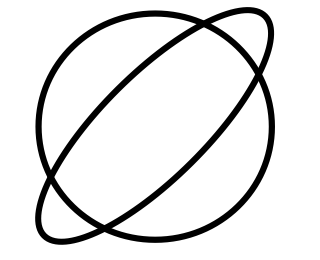}
     \label{fig:quadratic}
     }
     \subfigure[]{
     \includegraphics[width=.22\textwidth]{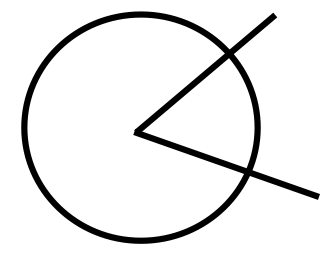}
     \label{fig:conic}
     }
     \caption{\edit{This figure illustrates constraints on our hypothesis space. These constraints arise from side knowledge available about a set of unlabeled examples. The $\ell_2$ balls in (a), (b), (c) and (d) represent coefficients of linear functions in two dimensions. (a) and (b) represent intersection of a ball and one or several half spaces. Theorems \ref{theorem:single-linear-constraint}, \ref{theorem:polygonal-constraints} and Proposition \ref{prop:single-linear-constraint-duality} analyze these situations. (c) shows the intersection of a ball and an ellipsoid. Theorems \ref{theorem:quadratic-rad-upper-bdd}, \ref{theorem:quadratic-rad-lower-bdd} and \ref{theorem:quadratic-rad-duality}  correspond to this setting. (d) shows the intersection of a ball with a second order cone. Theorem \ref{theorem:conic-bdd} corresponds to this setting.}\label{fig:side-knowledge}}
\end{figure}

Side knowledge can be particularly helpful in cases where data are scarce; these are precisely circumstances when data themselves cannot fully define the predictive model, and thus domain knowledge can make an impact in predictive accuracy. That said, for any type of side knowledge (sparsity, smoothness, and the side knowledge considered here), the examples and hypothesis space may not conform in reality to the side knowledge. (Similarly, the training data may not be truly random in practice.) However, if they do, we can claim lower sample complexities, and potentially improve our model selection efforts. Thus, we cannot claim that our side knowledge is always true knowledge, but we can claim that if it is true, we are able to gain some benefit in learning.

\subsection*{\edit{Motivating examples}}

\edit{
\citet{fung2002knowledge} added multiple linear constraints (polygonal constraints) to a specific ERM algorithm, the linear SVM, as a way to incorporate prior knowledge. They investigated the effect of using this type of prior knowledge for classification on a DNA promoter recognition dataset \cite{towell1990refinement}. In this classification task, the linear constraints result from precomputed rules that are separate from the training data (this is similar to our polygonal setting where constraints are generated from knowledge about the unlabeled examples). The ``leave-one-out'' error from the 1-norm SVM with the additional constraints was less than that of the plain 1-norm SVM and other training-data-based classifiers such as decision trees and neural networks. This and other types of knowledge incorporation in SVMs are reviewed by \citet{lauer2008incorporating} and also \citet{le2006simpler}. }

\edit{
\citet{classo} motivated the use of linear constraints with LASSO, which is also an ERM procedure. In their experiment, they estimated a demand probability function using an on-line auto lending dataset. 
They ensured monotonicity of the demand function by applying a set of linear constraints (similar to the poset constraints in \ref{subsec:linear}) and compared the output to two other methods: logistic regression and the unconstrained LASSO, both of which output non-monotonic demand probability curves.
}

\edit{\citet{nguyen2008classification} considered additional unlabeled examples whose labels are partially known. In particular, they worked on a type of multi-class classification task where they know that the label of each unlabeled example belongs to a known subset of the set of all class labels. This knowledge about the unlabeled examples translates into multiple linear constraints (polygonal constraints). They provided experimental results on five datasets showing improvements over multi-class SVMs.
}

\edit{
\citet{gomez2008semisupervised} implemented a technique (known as LapSVMs) that uses Laplacian regularization augmented with standard SVMs for two image classification tasks related to urban monitoring and cloud screening (which are both remote sensing tasks). Laplacian regularization means that the regularization term is a quadratic function of the model, derived from a set of unlabeled examples, like our quadratic setting (see Section \ref{subsec:quadratic}). In both tasks, the Laplacian-regularized linear SVMs outperformed the standard SVMs in terms of overall accuracy (these improvements are of the order of 2-3\% in both cases).
}

\edit{
\citet{shivaswamy2006second} formulated robust classification and regression problems as described in Section \ref{subsec:conic} leading to conic constraints on the model class. For classification, they used the OCR, Heart, Ionosphere and Sonar datasets from the UCI repository to illustrate the effect of missing values and how robust SVM classification (which introduces second order conic constraints) provides better classification accuracy than the standard SVM classifier after imputation. For regression, they showed improvements in prediction accuracy of a robust version of SVR (again introducing conic constraints on the hypothesis space) as compared to a standard SVR trained after imputation on the Boston housing dataset (also from the UCI repository).
}

\editnew{Finally, Appendix A also provides experimental results showing the advantage of using side knowledge in a ridge regression problem.}

\section{Linear, Polygonal, Quadratic and Conic Constraints}
\label{sec:structure}
We are given training sample $S$ of $n$ examples $\{(x_{i},y_{i})\}_{i=1}^{n}$ with each observation $x_{i}$ belong to a set $\X$ in $\mathbb{R}^{p}$. 
Let the label $y_{i}$ belong to a set $\Y$ in $\R$. 
In addition, we are given a set of $m$ unlabeled examples $\{\xt_{i}\}_{i=1}^{m}$.
We are not given the true labels $\{\yt_{i}\}_{i=1}^{m}$ for these observations.  Let $\F$ be the function class (set of hypotheses) of interest, from which we want to choose a function $f$ to predict the label of future unseen observations. 
Let it be linear, parameterized by coefficient vector $\beta$ and its description will change based on the constraints we place on $\beta$. 

Consider the empirical risk minimization problem: $\min_{f \in \F} \frac{1}{n}\sum_{i=1}^{n}\edit{l(f(x_{i}),y_{i})}$. \edit{Here the loss function is a Lipschitz continuous function such as the squared, exponential or hinge loss among others.} \edit{This supervised learning setup encompasses both supervised classification ($\Y$ is a discrete set) and regression ($\Y$ is equal to $\R$).}
Regularization on $f$ acts to enforce assumptions that the true model comes from a restricted class, so that $\F$ is now defined as \[\{ f |  f:\X \mapsto \Y, f(x) = \beta^{T}x, R_{l}(f) \leq c_{l} \textrm{ for } l=1,...,L \},\] where $()^{T}$ represents the transpose operation. Here we have appended $L$ additional constraints for regularization to the description of the hypothesis set $\F$. 
Especially if the training set is small, side knowledge can be very powerful in reducing the size of $\F$. Particularly if constants $\{c_{l}\}_{l=1}^{L}$ are small, the size of $\F$ be reduced substantially.

\subsection{Assumptions leading to linear \edit{and polygonal} constraints}\label{subsec:linear}
We will provide three settings to demonstrate that linear constraints arise in a variety of natural settings: poset, must-link, and sparsity on $\{\yt_{i}\}_{i=1}^{m}$. In all three, we will include standard regularization of the form $\|\beta\|_q\leq c_1$ by default.\\

\noindent \textbf{Poset}:
Partial order information about the labels $\{\yt_{i}\}_{i=1}^{m}$ can be captured via the following constraints: $f(\xt_{i}) \leq f(\xt_{j}) + c_{i,j}$ for any collection of pairs $(i,j) \in [1,...,m]\times[1,...,m]$. This gives us up to $m^2$ constraints of the form $\beta^{T}(\xt_{i} - \xt_{j}) \leq c_{i,j}.$
\edit{$\F$} can be described as: $\F:=\{ f |   f(x) = \beta^{T}x, \|\beta\|_{q} \leq c_{1}, \beta^{T}(\xt_{i} - \xt_{j}) \leq c_{i,j}, \forall (i,j) \in E\}$,  where $E$ is the set of pairs of indices of unlabeled data that are constrained.\\

\noindent \textbf{Must-link}: Here we bound the absolute difference of labels between pairs of unlabeled examples: $ |f(\xt_{i}) - f(\xt_{j})| \leq c_{i,j}$. This captures  knowledge about the nearness of the labels. This leads to two linear constraints:  $-c_{i,j} \leq \beta^{T}(\xt_{i}-\xt_{j}) \leq c_{i,j}.$ These constraints have been used extensively within the semi-supervised \cite{zhu05survey} and constrained clustering settings \cite{lu2004semi,basu2006probabilistic} as must-link or `in equivalence' constraints. 
For must-link constraints, $\F$ is defined as:
$ \F:=\{ f |   f(x) = \beta^{T}x, \|\beta\|_{q} \leq c_{1}, -c_{i,j} \leq \beta^{T}(\xt_{i}-\xt_{j}) \leq c_{i,j}, \forall(i,j) \in E\}$, where $E$ is again the set of pairs of indices of unlabeled data that are constrained.\\

\noindent \textbf{Sparsity and its variants on a subset of $\{\yt_{i}\}_{i=1}^{m}$:}
Similar to sparsity assumptions on $\beta$, here we want that only a small set of labels is nonzero among a set of unlabeled examples.
In particular, we want to bound the cardinality of the support of the vector $[\yt_{{1}} \hdots \yt_{{|\mathcal{I}|}}]$ for some index set $\mathcal{I} \subset \{1,...,m\}$. Such a constraint is nonlinear. Nonetheless, a convex constraint of the form $\|[\yt_{{1}} \hdots \yt_{{|\mathcal{I}|}}]\|_{1} \leq c_{\mathcal{I}} $ ($2^{|\mathcal{I}|}$ linear constraints) can be used as a proxy to encourage sparsity. The function class is defined as: $ \F:=\{ f |   f(x) = \beta^{T}x, \|\beta\|_{q} \leq c_{1}, \|[\beta^T\xt_{{1}} \hdots \beta^T\xt_{{|\mathcal{I}|}}]\|_{1} \leq c_{\mathcal{I}}\}$.
A similar constraint can be obtained if we instead had partial information \edit{with respect to} the dual norm: $\|[\yt_{{1}} \hdots \yt_{{ |\mathcal{I}| }}]\|_{\infty} \leq c_{\mathcal{I}}$.\\ 

\subsection{Assumptions leading to quadratic constraints}
\label{subsec:quadratic}
We will provide several settings to show that quadratic constraints arise naturally.\\

\noindent \textbf{Must-link:} 
A constraint of the form
$(f(\xt_{i}) - f(\xt_{j}))^{2} \leq c_{i,j}$ 
can be written as 
 $ 0 \leq \beta^{T}A \beta \leq c_{i,j}$ with $A = (\xt_{i}-\xt_{j})(\xt_{i}-\xt_{j})^T$. Here $A$ is rank-deficient as it is an outer product, which leads to an unbounded ellipse; however, its intersection with a full ellipsoid (for instance, an $\ell_{2}$-norm ball) is not unbounded and indeed can be a restricted hypothesis set.
Set $\F$ is defined by:
$\F = \{\beta: \beta^{T}\beta \leq c_{1}, \beta^{T} (\xt_{i}-\xt_{j})(\xt_{i}-\xt_{j})^T \beta \leq c_{i,j}; (i,j) \in E\}$, where $E$ is again the set of pairs of indices of unlabeled data that are constrained.\\

\noindent\textbf{Constraining label values for a pair of examples: }  We can define the following relationship between the labels of two unlabeled examples using quadratic constraints: if one of them is large in magnitude, the other is necessarily small. This can be encoded using the inequality: $f(\xt_{i})\cdot f(\xt_{j}) \leq c_{i,j}$. If $f(x) \in \Y \subset \R_{+}$, then $f(\xt_{i})\cdot f(\xt_{j}) \leq c_{i,j}$ gives the following quadratic constraint on $\beta$ with the associated rank 1 matrix being $A = \xt_{i}\xt_{j}^{T}$: $\beta^T A\beta \leq c_{i,j}.$ This is not quite an ellipsoidal constraint yet because matrices associated with ellipsoids are symmetric positive semidefinite. Matrix $A$ on the other hand is not symmetric. Nonetheless, the quadratic constraint remains intact when we replace matrix $A$ with the symmetric matrix $\frac{1}{2}(A + A^{T})$. If in addition, the symmetric matrix is also positive-definite (which can be verified easily), then this leads to an ellipsoidal constraint. The hypothesis space $\F$ becomes: $\F = \left\{\beta: \beta^{T}\beta \leq c_{1}, \beta^{T} \xt_{i}\xt_{j}^{T}\beta \leq c_{i,j}; (i,j) \in E \right\}.$\\

\noindent\textbf{Energy of estimated labels: } 
We can place an upper bound constraint on the sum of squares (the ``energy") of the predictions, which is: 
$||\XunlabT\beta||_{2}^{2} = \sum_{i}(\beta^{T}\xt_{i})^{2} = \beta^T(\sum_{i}\xt_{i}\xt_{i}^{T})\beta$ where $X_{U}$ is a $p \times m$ dimensional matrix with $\tilde{x}_i$'s as its columns.\footnote{Note that this notation is not the usual notation where observations $\tilde{x}_i$'s are stacked as rows.}
The set $\F$ is
$\F = \left\{\beta: \beta^{T}\beta \leq c_{1}, ||\XunlabT\beta||_{2}^{2} \leq c \right\}$.
Extensions like the use of Mahalanobis distance or having the norm act on only a subset of the estimates of $\{\yt\}_{i=1}^{m}$ follow accordingly.\\

\noindent\textbf{Smoothness and other constraints on $\{\yt_{i}\}_{i=1}^{m}$:}
Consider the general ellipsoid constraint $\|\Gamma\XunlabT\beta\|_{2}^{2} \leq c$ where we have added an additional transformation matrix $\Gamma$ in front of $\XunlabT\beta$. If $\Gamma$ is set to the identity matrix, we get the energy constraint previously discussed. If $\Gamma$ is a banded matrix with $\Gamma_{i,i} = 1$ and $\Gamma_{i,i+1} = -1$ for all $i=1,...,m$ and remaining entries zero, then we are encoding the side knowledge that the variation in the labels of the unlabeled examples is smoothly varying: we are encouraging the unlabeled examples with neighboring indices to have similar predicted values.
This matrix $\Gamma$ is an instance of a difference operator in the numerical analysis literature. 
In this context, banded matrices like $\Gamma$ model discrete derivatives.
By including this type of constraint, problems with identifiability and ill-posedness of an optimal solution $\beta$ are alleviated. That is, as with the Tikhonov regularization on $\beta$ in least squares regression, constraints derived from matrices like $\Gamma$ reduce the condition number. 
The set $\F$ is defined as: 
$\F = \left\{\beta: \beta^{T}\beta \leq c_{1}, \|\Gamma\XunlabT\beta\|_{2}^{2} \leq c \right\}.$\\

\noindent\textbf{Graph based methods:} Some graph regularization methods such as manifold regularization \cite{belkin2004semi} also encode information about the labels of the unlabeled data. They also lead to convex quadratic constraints on  $\beta$. 
Here, along with the unlabeled examples $\{\xt_{i}\}_{i=1}^{m}$, our side knowledge consists of an $m$-node weighted graph $G = (V,E)$ with the Laplacian matrix $L_{G} = D - A$. Here, $D$ is a $m\times m$-dimensional diagonal matrix with the diagonal entry for each node equal to the sum of weights of the edges connecting it. Further, $A$ is the adjacency matrix containing the edge weights $a_{ij}$, where $a_{ij} = 0$ if $(i,j) \notin E$ and $a_{ij} = e^{-c\|\xt_{i}-\xt_{j}\|_{q}}$ if $(i,j) \in E$ (other choices for the weights are also possible). 
The quadratic function $(\XunlabT\beta)^{T} L_{G}(\XunlabT\beta)$ is then twice the sum over all edges, of the weighted squared difference between the two node labels corresponding to the edge: 
$2\sum_{(i,j) \in E}a_{ij}\left(f(\xt_{i}) - f(\xt_{j})\right)^{2}.$
Intuitively, if we have the side knowledge that this quantity is small, it means that a node should have similar labels to its neighbors. For classification, this typically encourages the decision boundary to avoid dense regions of the graph. The set $\F$ is defined as:  $\F = \{\beta: \beta^{T}\beta \leq c_{1}, \beta^{T}\XunlabT L_{G}\XunlabT\beta \leq c\}$.

\subsection{ \edit{Assumptions leading to conic constraints}} 
\label{subsec:conic}
\edit{
We provide two scenarios that naturally lead to conic constraints on the model class: robustness against uncertainty and stochastic constraints.\\
}

\noindent \edit{\textbf{Robustness against uncertainty in linear constraints:} Consider any of the linear constraints considered in Section \ref{subsec:linear}. All of these can be generically represented as: $\{a_k^T\beta \leq 1\;\; \forall k=1,..,K\}$ where for each $k$, $a_k$ is a function of the unlabeled sample $\{\tilde{x}_j\}_{j=1}^{m}$ (for instance, $a_k = \tilde{x}_i - \tilde{x}_k$ for Poset constraints). Further assume that each $a_k$ is only known to belong to an ellipsoid $\Xi_{k} = \{\overline{a}_k + A_ku: u^Tu \leq 1\}$ with both parameters $\overline{a}_k$ and $A_k$ known. This can happen due to measurement limitations, noise and other factors. We want to guarantee that, irrespective of the true value of $a_k \in \Xi_k$, we still have $a_k^T\beta \leq 1$.
 }
 
 \edit{
Borrowing a trick used in the robust linear programming literature, we can encode \cite{lanckriet2003robust} the above requirement succinctly as:
\begin{align*}
\overline{a}_k^T\beta + \|A_k^T\beta\|_2 \leq 1, \forall k=1,...,K
\end{align*}
which is a set of second-order cone constraints. The feasible set becomes smaller when the linear constraints $\{a_k^T\beta \leq 1\; \forall k=1,...,K\}$ are replaced with the conic constraints above.\\
}

\noindent  \edit{\textbf{Stochastic Programming:} Consider a probabilistic constraint of the form $\mathbb{P}_{a_k}(a_k^T\beta \leq 1) \geq \eta_k$, where $a_k$ is now considered a random vector. The motivation for $a_k$ is the same as before (see Section \ref{subsec:linear}). If we know that $a_k$ is normally distributed (for instance, due to additive noise) with mean $\overline{a}_k$ and covariance matrix $B_k$, then the probabilistic constraint is the same as: $\overline{a}_k^T\beta + \Phi^{-1}(1 - p)\|B_k^{1/2}\beta\|_2 \leq 1$, where $\Phi^{-1}()$ is the inverse error function. To see this, let $u_k = a_k^T\beta$ be a scalar random variable with mean $\overline{u_k}$ and variance $\sigma_k^2$ (this is equal to $\beta^T B_k\beta$). Then, our original constraint can be written as $\mathbb{P}\left(\frac{u_k-\overline{u}_k}{\sigma_k} \leq \frac{1-\overline{u}_k}{\sigma_k}\right) \geq \eta_k$. Since $\frac{u_k-\overline{u}_k}{\sigma_k}\sim \mathcal{N}(0,1)$, we can rewrite our constraint as: $\Phi\left(\frac{1-\overline{u}_k}{\sigma_k}\right) \geq \eta_k$ where 
$\Phi(z)$ is the cumulative distribution function for the standard normal. 
Further $\Phi\Big(\frac{1-\overline{u}_k}{\sigma_k}\Big) \geq \eta_k$ implies $ \frac{1-\overline{u}_k}{\sigma_k} \geq \Phi^{-1}(\eta_k)$. Rearranging terms, we get $\overline{u}_k + \Phi^{-1}(\eta_k)\sigma_k \leq 1$. Finally, substituting the values for $\overline{u}_k$ and $\sigma_k$ gives us the following constraint:
\begin{align*}
\overline{a}_k^T\beta +  \Phi^{-1}(\eta_k)\|B_k^{1/2}\beta\|_2 \leq 1,
\end{align*}
which is a second order conic constraint~\cite{lobo1998applications}.
}

\edit{
\begin{remark}
A question of practical interest would be about ways to impose constraints seen in Sections \ref{subsec:linear}, \ref{subsec:quadratic} and \ref{subsec:conic} in a computationally efficient manner.  Fortunately, for all the cases we have considered thus far, the side knowledge can be encoded as a set of convex constraints leading to efficient algorithms (if the original empirical risk minimization problem is convex). Further, note that unlike must-link and similarity side knowledge that lead to convex constraints, cannot-link and dissimilarity knowledge is relatively harder to impose and is typically non-convex. 
\end{remark} 
}

\section{Generalization Bounds}
\label{sec:bound}

\edit{
In each of the scenarios considered in Section \ref{sec:structure}, essentially we are given $m$ unlabeled examples $\xt$ whose subsets satisfy various properties or side knowledge (for instance, linear ordering, quadratic neighborhood similarity, etc). This side knowledge is also shown to constrain the hypothesis space in various ways. In this section, we will attempt to answer the following statistical question: what effect do these constraints have on the generalization ability of the learned model?
}
We will compute bounds on the complexity of the hypothesis space when the types of constraints seen in Section \ref{sec:structure} are included.


\subsection{Definition of Complexity Measures}

We will look at two complexity measures: the covering number of a hypothesis set and the Rademacher complexity of a hypothesis set. Their definitions are as follows:
\begin{definition} \textit{Covering Number \cite{kol61}:} Let $A \subseteq \Omega$ be an arbitrary set and $(\Omega, \rho)$ a (pseudo-)metric space. Let $|\cdot|$ denote set size. 
For any $\epsilon > 0$, an \textbf{$\epsilon$-cover} for $A$ is a finite set $U \subseteq \Omega$ (not necessarily $ \subseteq A$) s.t. $ \forall \omega \in A, \exists u \in U$ with $d_{\rho}(\omega, u) \leq \epsilon$. 
The \textbf{covering number} of $A$ is $N(\epsilon,A,\rho) := \inf_{U} |U|$ where $U$ is an $\epsilon$-cover for $A$.
\end{definition}
\begin{definition} \textit{Rademacher Complexity \cite{bartlett02}:} Given a training sample $S = \{x_{1},...,x_{n}\}$, with each $x_i$ drawn i.i.d. from $\mu_{\mathcal{X}}$, and hypothesis space $\F$, $\F_{|S}$ is the defined as the restriction of $\F$ with respect to $S$. The \textit{empirical Rademacher complexity of $\F_{|S}$} 
is
\begin{equation*}
\Rad = \mathbb{E}_{\sigma}\left[\sup_{f \in \F} \frac{1}{n}\edit{\left| \sum_{i=1}^{n}\sigma_{i}f(x_{i}) \right|} \right]
\end{equation*}
where $\{\sigma_i\}$ are Rademacher random variables ($\sigma_i = 1$ with probability $1/2$ and $\sigma_i =-1$ with probability $1/2$). The \textit{Rademacher complexity of $\F$} is its expectation: 
\[\mathcal{R}(\F) = \mathbb{E}_{S\sim (\mu_{\mathcal{X}})^{n}}[\Rad].\]
\end{definition}
If instead we let $\sigma_{i} \sim \mathcal{N}(0,1)$ in the definition, this is the Gaussian complexity of the function class. Generalization bounds often use both these quantities in their statements \cite{bartlett02}. \editnew{Unless otherwise specified, the feature vectors in feature space $\X$ are bounded in norm by constant $X_b > 0$ and the coefficient vectors of the linear function class $\F$ are bounded in norm with constant $B_b > 0$.}

\subsection{\edit{Complexity measures within generalization bounds}} 
\edit{
Given these definitions, a generalization bound statement can be written as follows \cite{bartlett02}: With probability at least $1-\delta$ over the training sample $S$,
\begin{align*}
\forall f \in \F,\;\; \E_{x,y}[l(f(x),y)] \leq \frac{1}{n}\sum_{i=1}^{n}l(f(x_i),y_i) + 4\mathcal{L}\Rad + O\left(\sqrt{\frac{\log \frac{1}{\delta}}{2n}}\right),
\end{align*}
where $\mathcal{L}$ is the Lipschitz constant of the loss function $l$. A relation between the empirical Rademacher complexity and covering number can be used to state the above uniform convergence statement in terms of the covering number. The relation (also known as Dudley's entropy integral) is \cite{talagrand2005generic}:
\begin{align*}
\Rad \leq c\int_{0}^{\infty}\sqrt{\frac{\log N(\sqrt{n}\epsilon,\F_{|S},\|\cdot\|_2)}{n}}d\epsilon,
\end{align*}
where $\F_{|S} = \{(f(x_{1}),\hdots,f(x_{n})): f \in \F\}$ and $c$ is a constant. 
Thus, we study upper bounds for covering numbers and empirical Rademacher complexities interchangeably through the rest of the paper.
}

\subsection{Complexity results with a single linear constraint}
\label{subsec:linear-bdd}

\edit{We state two results: the first is based on volumetric arguments and bounds the covering number and the second is based on convex duality and bounds the empirical Rademacher complexity. The first is a result from \citet{TuRu14progress} while the second is new to this paper.\\
}

\noindent\edit{\textbf{Volumetric upper bound on the covering number:} \citet{TuRu14progress} analyzed the setting where a bounded linear function class is further constrained by a half space. The motivation there was to study a specific type of side knowledge, namely knowledge about the cost to solve a decision problem associated with the learning problem.  The result there extends well beyond operational costs and is applicable to our setting where we have a $\ell_2$ bounded linear function class with a single half space constraint.
}

\begin{theorem} \citep[Theorem 2 of][]{TuRu14progress} Let $\X = \{x \in \mathbb{R}^{p}: \|x\|_{2} \leq X_{b}\}$ \editnew{with $X_b > 0$,} and let $\mu_{\X}$ be the marginal probability measure on $\X$. Let $$\allowbreak \F = \left\{f | f:\X\mapsto \Y, f(x) = \beta^{T}x, \|\beta\|_{2} \leq B_{b},\;  a^{T}\beta \leq 1 \right\},$$\editnew{ with $B_b > 0$}. \edit{Let $\F_{|S} = \{(f(x_{1}),\hdots,f(x_{n})): f \in \F\}$}. Then for all $\epsilon > 0$, for any sample $S$,
\begin{align*}
\edit{N(\sqrt{n}\epsilon,\F_{|S}, \| \cdot\| _{2}) \leq \alpha(p,a,\epsilon)\left(\frac{2B_{b}X_b}{\epsilon} + 1\right)^{p} .} 
\end{align*}
Also, defining $r = B_{b} + \frac{\epsilon}{2X_b}$ and \edit{$V_{p}(r) = \frac{\pi^{p/2}}{\Gamma(p/2 + 1)}r^{p}$}, 
the function $\alpha$ above is:
\begin{align*}
\alpha(p,a,\epsilon)& =\\
 1 - \frac{1}{V_{p}(r)}&\int_{\theta = \cos^{-1}\left(\frac{\|a\|_{2}^{-1} + \frac{\epsilon}{2X_b}}{r}\right)}^{0}V_{p-1}(r\sin\theta)d(r\cos\theta).
\end{align*}
\label{theorem:single-linear-constraint}
\end{theorem}

\edit{\textit{Intuition:}} The function $\alpha(p,a,\epsilon)$ can be considered to be the normalized volume of the ball (which is 1) minus the portion that is the spherical cap cut off by the linear constraint. It comes directly from formulae for the volume of spherical caps. We are integrating over the volume of a $p-1$ dimensional sphere of radius $r\sin\theta$ and the height term is $d(r\cos\theta)$.

This bound shows that the covering number bound can depend on $a$, which is a direct function of the unlabeled examples $\{\xt_{i}\}_{i=1}^{m}$. As the norm $\|a\|_2$ increases, $\|a\|_2^{-1}$ decreases, thus $\alpha(p,a,\epsilon)$ decreases, and the whole bound decreases. This is a mechanism by which side information on the labels of the unlabeled examples influences the complexity measure of the hypothesis set, potentially improving generalization. 

\edit{\textit{Relation to standard results:}} It is known \cite{kol61} that \editnew{set} $\mathcal{B} = \{\beta:  \| \beta\| _{2} \leq B_{b}\}$ \editnew{(with $B_b > 0$ being a fixed constant as before)} has a bound on its covering number of the form \edit{$N(\epsilon,\mathcal{B},\| \cdot\| _{2}) \leq \left(\frac{2B_{b}}{\epsilon} + 1\right)^{p}$}. 
Since in Theorem \ref{theorem:single-linear-constraint} the same \edit{term} appears, multiplied by a factor \edit{that is at most one and that can be substantially less than one}, the bound in Theorem \ref{theorem:single-linear-constraint} can be tighter. \\

\edit{The above result bounds the covering number complexity for the hypothesis set. Next, we will bound the empirical Rademacher complexity for the same hypothesis set as above.\\}

\noindent\edit{\textbf{Convex duality based upper bound on empirical Rademacher complexity:} Consider the setting in Theorem \ref{theorem:single-linear-constraint}. \editnew{Let $x_i \in \X = \{x: \|x\|_2 \leq X_b\}$ for $i=1,...,n$ as before.} Our attempt to use convex duality to upper bound empirical Rademacher complexity yields the following bound.
\begin{proposition}
Let $\X = \{x \in \mathbb{R}^{p}: \|x\|_{2} \leq X_{b}\}$ \editnew{with $X_b > 0$} and $$\allowbreak \F = \left\{f | f:\X\mapsto \Y, f(x) = \beta^{T}x, \|\beta\|_{2} \leq B_{b},  a^{T}\beta \leq 1 \right\},$$ \editnew{with $B_b > 0$.} Then,
\begin{align*}
\Rad \leq \max\left(\E_{\sigma}\left[\min_{\eta \geq 0} (B_b\|X_L\sigma-\eta a\|_2 + \eta) \right],\E_{\sigma}\left[\min_{\eta \geq 0} (B_b\|X_L\sigma+\eta a\|_2 + \eta)\right]\right),
\end{align*}
where $X_L = [x_1\; \hdots \;x_n]$ is a $p\times n$ dimensional feature matrix and $\sigma$ is a $n\times 1$ dimensional vector of Bernoulli random variables taking values in $\{-1,1\}$.
\label{prop:single-linear-constraint-duality}
\end{proposition}
}

\edit{\textit{Intuition:} We can understand the effect of the linear constraint on the upper bound through the magnitude of vector $a$. Without loss of generality, let the expectation of the optimal value of the first minimization problem be higher (both minimization problems are structurally similar to each other except for a sign change within the norm term). For a fixed value of $\sigma$, this minimization problem involves the distance of vector $\Xlab\sigma$ to the scaled vector $a$ in the first term and the scaling factor $\eta$ itself as the second term. 
}
\edit{Thus, generally, if $\|a\|_2$ is large, the scaling factor $\eta$ can be small, resulting in a lower optimal value. We also know that larger $\|a\|_2$ corresponds to a tighter half space constraint. Thus, as the linear constraint on the hypothesis space becomes tighter, it makes the optimal solution $\eta$ and the optimal value smaller for each $\sigma$ vector. As a result, it tightens the upper bound on the empirical Rademacher complexity.
}

 
\edit{\textit{Relation to standard results:} An upper bound on each term of the $\max$ operation above can be found by setting $\eta = 0$ that recovers the standard upper bound of $\frac{B_b\sqrt{\textrm{trace}(X_L^TX_L)}}{\sqrt{n}}$ or $\frac{B_bX_b}{\sqrt{n}}$ without capturing the effect of the linear constraint $a^T\beta \leq 1$.
}

\subsection{\edit{Complexity results with polygonal/multiple linear constraints and general norm constraints}}
\label{subsec:multiple-linear-bdd}

\edit{The following result is from \citet{TulabandhulaRu13}, where the authors analyze the effect of decision making bias on generalization of learning. Again, as in the single linear constraint case, the result extends beyond the setting considered in that paper. In particular, it covers all the motivating scenarios described in Section \ref{subsec:linear}.
}

Let us define the matrix $[x_{1}\;\hdots \; x_{n}]$ as matrix $\Xlab$ where $x_i \in \X = \{x: \|x\|_r \leq X_b\}$ \editnew{and $X_b > 0$}. Then, $\Xlab^T$ can be written as $[h_{1}{ } \cdots{ }h_{p} ]$ with $h_{j} \in \mathbb{R}^{n},j=1,...,p$. Define function class $\F$ as 
\begin{align*}
\F=&\Big\{f | f(x) = \beta^{T}x, \beta \in \mathbb{R}^{p}, \| \beta\| _q \leq B_{b},\\
& \sum_{j=1}^{p}c_{j\nu}\beta_{j} +\delta_{\nu} \leq 1, \delta_{\nu} > 0, \nu=1,...,V\Big\},
\end{align*} 
where $1/r + 1/q = 1$ and $\{c_{j\nu}\}_{j,\nu}$, $\{\delta_{\nu}\}_{\nu}$ and \editnew{$B_{b} > 0$} are known constants.  \edit{In other words, we have $V$ linear constraints in addition to a $\ell_q$ norm constraint.}
\edit{As before, let $\F_{|S}$ be the restriction of $\F$ with respect to $S$. }

Let $\{\tilde{c}_{j\nu}\}_{j,\nu}$ be proportional to $\{c_{j\nu}\}_{j,\nu}$ \edit{in the following manner}:
\begin{eqnarray*}
\tilde{c}_{j\nu} &:=& \frac{c_{j\nu}n^{1/r}X_{b}B_{b}}{\| h_{j}\| _{r}} \;\;\textrm{ } \forall j=1,...,p \textrm{ and } \nu=1,...,V.\\
\end{eqnarray*}
Let $K$ be a positive number. Further, let the sets $P^{K}$ parameterized by $K$ and $P_{c}^{K}$ parameterized by $K$ and $\{\tilde{c}_{j\nu}\}_{j,\nu}$ be:
$P^{K} := \left\{(k_{1},...,k_{p}) \in \mathbb{Z}^{p}: \sum_{j=1}^{p}|k_{j}| \leq K\right\},\nonumber$ and 
$P_{c}^{K} := \left\{(k_{1},...,k_{p}) \in P^{K}: \sum_{j=1}^{p}\tilde{c}_{j\nu}k_{j} \leq K \; \forall \nu = 1,...,V\right\}.$
Let $|P^{K}|$ and $|P_{c}^{K}|$ be the sizes of the sets $P^{K}$ and $P_{c}^{K}$ respectively. The subscript $c$ in $P_{c}^{K}$ denotes that this polyhedron is a constrained version of $P^{K}$.
Define $\matXbar$ to be equal to \edit{the product of} a diagonal matrix (whose $j^{th}$ diagonal element is $\scalej$) and $\Xlab$.
Define $\lambda_{\min}(\matXbar\matXbar^{T})$ to be the smallest eigenvalue of the matrix $\matXbar\matXbar^{T}$.
\begin{theorem}\citep[Theorem 6 of][]{TulabandhulaRu13} 
\begin{equation*}
N(\sqrt{n}\epsilon,\F_{|S},\| \cdot\| _{2}) \leq 
\begin{cases}
\min\{|P^{K_{0}}|,|P_{c}^{K}|\} & \textrm{if } \epsilon < X_{b}B_{b} \\
1 & \textrm{ otherwise}
\end{cases},
\end{equation*}
where $ K_{0} = \ceil*{\frac{X_{b}^{2}B_{b}^{2}}{\epsilon^{2}}}$
and $K$ is the maximum of $K_{0}$ and 
$$
\ceil*{\frac{nX_{b}^{2}B_{b}^{2}}{\lambda_{\min}(\matXbar\matXbar^{T})\Big[\min_{\nu=1,...,V} \frac{\delta_{\nu}}{\sum_{j=1}^{p}|\tilde{c}_{j\nu}|}\Big]^{2}}}.$$
\label{theorem:polygonal-constraints}
\end{theorem}

\edit{\textit{Intuition:}} The linear assumptions on the labels of the unlabeled examples $\{\xt_{i}\}_{i=1}^{m}$ determine the parameters $\{\tilde{c}_{j\nu}\}_{j,\nu}$ that in turn influence the complexity measure bound.  \edit{In particular, as the linear constraints given by the $c_{j\nu}$'s force the hypothesis space to be smaller, they force $|P_{c}^{K}|$ to be smaller. This leads to a tighter upper bound on the covering number.}

\edit{\textit{Relation to standard results:} We recover the covering number bound for linear function classes given in \cite{zhang02} when there are no linear constraints. In this case, the polytope $P^{K}$ is well structured and the number of integer points in it can be upper bounded in an explicit way combinatorially. 
}

\edit{It is possible to convex duality to upper bound the empirical Rademacher complexity as we did in Proposition \ref{prop:single-linear-constraint-duality}. However, the intuition is less clear, and thus, we omit the bound here.}

\subsection{Complexity results with quadratic constraints}
\label{subsec:quadratic-bdd}

Consider the set $\F = \{f: f=\beta^{T} x, \beta^{T}A_{1}\beta \leq 1, \beta^{T}A_{2} \beta \leq 1 \}.$ Assume that at least one of the matrices is positive definite and both are positive-semidefinite, symmetric. 
Let $\Xi_{1} = \{\beta: \beta^{T}A_{1}\beta \leq 1\}$ and $\Xi_{2} = \{\beta: \beta^{T}A_{2}\beta \leq 1\}$ be the corresponding ellipsoid sets. \\

\noindent\edit{\textbf{Upper bound on empirical Rademacher complexity:}} 
We first find an ellipsoid $\Xi_{\textrm{int}\gamma}$ (with matrix $A_{\textrm{int}\gamma}$) circumscribing the intersection of the two ellipsoids $\Xi_{1}$ and $\Xi_{2}$ and then find a bound on the Rademacher complexity of a corresponding function class leading to our result for the quadratic constraint case. We will pick matrix $A_{\textrm{int}\gamma}$ to have a particularly desirable property, namely that it is \textit{tight}. We will call a circumscribing ellipsoid \textit{tight} when no other ellipsoidal boundary comes between its boundary and the intersection ($\Xi_{1}\cap\Xi_{2}$). If we thus choose this property as our criterion for picking the ellipsoid, then according to the following result, we can do so by a convex combination of the original ellipsoids:

\begin{theorem} \citep[Circumscribing ellipsoids,][]{kahan1968circumscribing}
 There is a family of circumscribing ellipsoids that contains every tight ellipsoid. Every ellipsoid $\Xi_{\textrm{int}\gamma}$ in this family has $\Xi_{\textrm{int}\gamma} \supseteq (\Xi_{1}\cap \Xi_{2})$ and is generated by matrix $A_{\textrm{int}\gamma} = \gamma A_{1} + (1-\gamma) A_{2}$, $\gamma \in [0,1]$.
 \label{theorem:kahan}
\end{theorem}
Using the above theorem, we can find a tight ellipsoid $\{\beta: \beta^{T}A_{\textrm{int}\gamma}\beta \leq 1\}$ that contains the set $\{ \beta: \beta^{T}A_{1}\beta \leq 1, \beta^{T}A_{2} \beta \leq 1\}$ easily. Note that the right hand sides of the quadratic constraints defining these ellipsoids can be equal to one without loss of generality.
\begin{theorem} (Rademacher complexity of linear function class with two quadratic constraints) Let $$\F = \{f: f(x)=\beta^{T} x: \beta^{T}\mathbb{I}\beta \leq B_{b}^{2}, \beta^{T}A_{2} \beta \leq 1\}$$ with $A_{2}$ symmetric positive-semidefinite \editnew{and $B_b > 0$}. Then,
\begin{align}
\Rad \leq \frac{1}{{n}}\sqrt{\textrm{trace}(\Xlab^{T}A_{\textrm{int}\gamma}^{-1}\Xlab)},
\label{eqn:quadratic-rad-upper-bdd}
\end{align}
where $A_{\textrm{int}\gamma}$ is the matrix of a circumscribing ellipsoid $\{\beta: \beta^{T}A_{\textrm{int}\gamma}\beta \leq 1\}$ of the set $\{ \beta: \beta^{T}\mathbb{I}\beta \leq B_{b}^{2}, \beta^{T}A_{2} \beta \leq 1\}$ and $\Xlab$ is the matrix $[x_1\; \hdots\; x_n]$ with examples $x_i$'s as its columns.
\label{theorem:quadratic-rad-upper-bdd}
\end{theorem}

\edit{\textit{Intuition:} If the quadratic constraints are such they correspond to small ellipsoids, then the circumscribing ellipsoid will also be small. Correspondingly, the eigenvalues of $\Aint$ will be large. Since, the upper bound depends inversely on the magnitude of these eigenvalues (since it depends on $\Aint^{-1}$), it becomes tighter. Also, in the setting where the original ellipsoids are large and elongated but their intersection region is small and can be bounded by a small circumscribing ellipsoid, the upper bound is again tighter.
}

\edit{\textit{Relation to standard results:}} If $A_{\textrm{int}\gamma}$ is diagonal (or axis-aligned), then we can write the empirical complexity $\Rad$ in terms of the eigenvalues $\{\lambda_{i}\}_{i=1}^{p}$ as 
$\Rad \leq \frac{1}{n}\sqrt{\sum_{j=1}^{n}\sum_{i=1}^{p}\frac{x_{ji}^{2}}{\lambda_{i}}}$ and this can be bounded by $\frac{X_{b}B_{b}}{\sqrt{n}}$  \cite{kakade2008complexity} when $A_{2} = \mathbf{0}$. In that case, all of the $\lambda_i$ are $\frac{1}{B_{b}^{2}}$.
  
\begin{remark} Since we can choose any circumscribing matrix $A_{\textrm{int}\gamma}$ in this theorem, we can perform the following optimization to get a circumscribing ellipsoid that minimizes the bound: 
\begin{align}
\label{eqn:quadratic-rad-bound-minimize}
\min_{\gamma \in [0,1]} \textrm{trace} (\Xlab^{T}(\gamma A_{1} + (1-\gamma)A_{2})^{-1}\Xlab). 
\end{align}
This optimization problem is a univariate non-linear program. \\
\end{remark}
 
\noindent\edit{\textbf{Lower bound on empirical Rademacher complexity:}} 
We will now show that the dependence of the complexity on $A_{\textrm{int}\gamma}^{-1}$ is near optimal. 

Since $\Aint$ is a real symmetric matrix, \edit{let us} decompose $\Aint$ into a product $P^T DP$ where $D$ is a diagonal matrix with the eigenvalues of $\Aint$ as its entries and $P$ is an orthogonal matrix (i.e., $P^T P=I$). Our result, \edit{which is similar in form to the upper} bound of Theorem \ref{theorem:quadratic-rad-upper-bdd}, is as follows.
\begin{theorem}\label{theorem:quadratic-rad-lower-bdd}
\begin{align*}
\Rad \geq \frac{\kappa}{n\log n}\sqrt{\textrm{trace}(\Xlab^{T}A_{\textrm{int}\gamma}^{-1} \Xlab)}
\end{align*}
where
\begin{align*}
\kappa = \frac{\edit{1}}{C\sqrt{1 + \frac{2\pi pnX_b^2}{(\min_{j=1,...,p}\|(P\Xlab)_{j}\|_{2})^{2}}}},
\end{align*}
 $C$ is the constant in Lemma \ref{lemma:gaussian-rademacher}, $P$ is the orthogonal matrix from the decomposition of \editnew{matrix} $\Aint$ \editnew{defined in Theorem \ref{theorem:quadratic-rad-upper-bdd}}, $p$ \editnew{and $X_b > 0$} are problem constants, \editnew{$\Xlab$ is the matrix $[x_1\; \hdots\; x_n]$ with examples $x_i$'s as its columns,} and $n$ is the number of training examples.
\end{theorem}

\edit{\textit{Intuition:} The lower bound is showing that the dependence on $\sqrt{\textrm{trace}(\Xlab^T\Aint^{-1}\Xlab)}$ is tight modulo a $\log n$ factor and a factor ($\kappa$). The $\log n$ factor is essentially due to the use of the relation between Gaussian and Rademacher complexities in our proof technique. On the other hand, $\kappa$ depends on the interaction between the side knowledge about the  unlabeled examples (captured through matrix $P$) and the feature matrix $\Xlab$. If there is no interaction, that is, $P\Xlab$ has zero valued rows for all $j=1,...,p$, then the lower bound on empirical Rademacher complexity becomes equal to 0. On the other hand, when there is higher interaction between $\Aint$ (or equivalently, $P$) and $\Xlab$, then the factor $\kappa$ grows larger, tightening the lower bound on the empirical Rademacher complexity. 
}

\edit{ The dependence of the lower bound on the strength of the additional convex quadratic constraint is captured via $\Aint$ and behaves in a similar way to the upper bound. That is, when the constraint leads to a small circumscribing ellipsoid, the eigenvalues of $\Aint^{-1}$ are small and the lower bound is small (just like the upper bound). On the other hand, if the constraint leads to a larger circumscribing ellipsoid, the eigenvalues of $\Aint^{-1}$ are large, leading to a higher values of the lower bound (the upper bound also increases similarly).
}

\edit{\textit{Relation to standard results:} As with the upper bound, when there is no second quadratic constraint, $\Aint = \frac{1}{B_b^2}\mathbb{I}$. The  lower bound depends on the training data through the term $\sqrt{\textrm{trace}(\Xlab^T\Xlab)}$ in this case.}

\edit{\textit{Comparison to the upper bound:}} For comparison, we see that the upper bound in Theorem \ref{theorem:quadratic-rad-upper-bdd} is of the form $\allowbreak \frac{1}{n}\sqrt{\textrm{trace}(\Xlab^{T}A_{\textrm{int}\gamma}^{-1} \Xlab)}$ while the lower bound of Theorem \ref{theorem:quadratic-rad-lower-bdd} is of the form $$\allowbreak \frac{\kappa}{n\log n} \sqrt{\textrm{trace}(\Xlab^{T}A_{\textrm{int}\gamma}^{-1} \Xlab)},$$ where $\kappa$ \edit{depends} on $A_{\textrm{int}\gamma}$ and $\Xlab$.

The proof for the lower bound is similar to what one would do for estimating the complexity of a ellipsoid itself (without regard to a corresponding linear function class). See also the work of ~\citet{wainwrightNotes} for handling single ellipsoids.\\

\noindent\textbf{Comparison \edit{of empirical Rademacher complexity upper bound} with a covering number based bound:}
When matrix $A_{\textrm{int}\gamma}$ describing a circumscribing ellipsoid has eigenvalues $\{\lambda_{i}\}_{i=1}^{p}$, \edit{then} the covering number can be bounded as:
\begin{align*}
\edit{N(\sqrt{n}\epsilon,\F_{|S}, \| \cdot\| _{2}) \leq \Pi_{i=1}^{p}\left(\frac{2X_b}{\epsilon\sqrt{\lambda_{i}}} + 1\right).}
\end{align*}
To get a tight bound, among all circumscribing ellipsoids, we should pick one \edit{that} minimizes the right hand side of the bound. To do this, we solve an optimization problem involving volume minimization that is different than in (\ref{eqn:quadratic-rad-bound-minimize}). 
For instance, this volume minimization can be done using the following steps if at least one of the matrices among $A_1$ and $A_2$ is positive-definite:
\begin{itemize}[noitemsep,topsep=0pt,parsep=0pt,partopsep=0pt,leftmargin=*]
\item First, $A_{1}$ and $A_{2}$ are simultaneously diagonalized by congruence (say with a non-singular matrix called $C$) to obtain diagonal matrices $\textrm{Diag}(a_{1i})$ and $\textrm{Diag}(a_{2i})$.  We can guarantee that the set of ratios $\{\frac{a_{1i}}{a_{2i}}\}$ obtained will be unique.
\item The desired ellipsoid $A_{\textrm{int}\gamma^*}$ can then be obtained by computing
\begin{equation*}
\gamma^* \in \arg\max_{\gamma \in [0,1]} \Pi_{i=1}^{p}(\gamma a_{1i} + (1-\gamma)a_{2i})
\end{equation*}
 and then multiplying the optimal diagonal matrix $\textrm{Diag}(\gamma^* a_{1i} + (1-\gamma^*)a_{2i})$ with the congruence matrix $C$ appropriately. Optimal $\gamma^*$ can be found in polynomial time (for example, using Newton-Raphson). \\
\end{itemize}

\noindent\edit{\textbf{Comparison with the duality approach to upper bounding empirical Rademacher complexity:}
A convex duality based upper bound can be derived as shown below.
\begin{theorem} Consider the setting of Theorem \ref{theorem:quadratic-rad-upper-bdd}. Then, 
\begin{align}
\Rad \leq  \inf_{\eta \in [0,1]}\left\{ \frac{1}{4n} \textrm{trace}(X_L^TA_{\textrm{int}\eta}^{-1}X_L) +  \frac{1}{n}(B_b^2 + \eta(1-B_b^2))\right\},
\label{eqn:quadratic-rad-duality}
\end{align}
where $A_{\textrm{int}\eta} = \mathbb{I} + \eta(A_2 - \mathbb{I})$.
\label{theorem:quadratic-rad-duality}
\end{theorem}
This upper bound looks similar to the result in Equation (\ref{eqn:quadratic-rad-upper-bdd}). 
Note that $A_{\textrm{int}\eta}$ is different from $\Aint$ in Theorem \ref{theorem:quadratic-rad-upper-bdd}. $\Aint$ comes from a circumscribing ellipsoid, whereas $A_{\textrm{int}\eta}$ does not. 
Instead, the matrix $A_{\textrm{int}\eta}$ is picked such that $\eta$ minimizes the right hand side of the bound in Equation \ref{eqn:quadratic-rad-duality}. Qualitatively, we can see that if the matrix $A_2$ corresponding to the second ellipsoid constraint has large eigenvalues (for instance, when the second ellipsoid is a smaller sphere, or is an elongated thin ellipsoid), then $A_{\textrm{int}\eta}^{-1}$ is `small' (the eigenvalues are small) leading to a tighter upper bound on the empirical Rademacher complexity.
}

\noindent\edit{\textbf{Extension to multiple convex quadratic constraints:} Although Section \ref{subsec:quadratic-bdd} deals with only two convex quadratic constraints, the same strategy can be used to upper bound the complexity of hypothesis class constrained by multiple convex quadratic constraints. In particular, let $\F = \{f: f=\beta^{T} x, \beta^{T}A_{k}\beta \leq 1 \;\;\forall k=1,...,K \}$. Again, assume one of the matrices $A_k$ is positive definite. We can approach this problem in two stages. In the first step, we find an ellipsoid $\Xi_{\textrm{int}\gamma}$ (with matrix $\Aint$) circumscribing the intersections of the $K$ original ellipsoids and in the second step, we reuse Theorem \ref{theorem:quadratic-rad-upper-bdd} to obtain an upper bound in $\Rad$.
}

\edit{
We will generalize Equation (\ref{eqn:quadratic-rad-bound-minimize}) to look for a circumscribing ellipsoid from the family of ellipsoids parameterized by a $K$ dimensional vector $\gamma$ constrained to the $K-1$ simplex. In other words, the family of circumscribing ellipsoids is given by $\{\beta^T\Aint\beta \leq 1: \Aint = \sum_{k=1}^{K}\gamma_kA_k, \sum_{k=1}^{K}\gamma_k = 1, \gamma_k \geq 0 \;\;\forall k=1,...,K\}$. We can pick one circumscribing ellipsoid from this family by minimizing the right hand side of Equation \ref{eqn:quadratic-rad-upper-bdd} over the $K-1$ simplex similar to Equation (\ref{eqn:quadratic-rad-bound-minimize}):
\begin{align*}
\label{eqn:quadratic-rad-bound-minimize}
\min_{\gamma \in \left\{\gamma: \sum_{k=1}^{K}\gamma_k = 1, \gamma_k \geq 0 \;\;\forall k=1,...,K\right\}} \textrm{trace} \left(\Xlab^{T}\left(\sum_{k=1}^{K}\gamma_kA_k\right)^{-1}\Xlab\right). 
\end{align*}
The above optimization problem is a $K-1$ dimensional polynomial optimization problem.
}

\subsection{Complexity results with linear and quadratic constraints}
\label{subsec:linear-quadratic-bdd}
Consider now the setting where we have both linear and quadratic constraints. In particular, we can have the assumptions leading to linear constraints and those leading to quadratic constraints hold simultaneously. In such a setting, based on Theorems \ref{theorem:polygonal-constraints} and \ref{theorem:kahan}, we can get a potentially tighter covering number result as follows. Let $x_i \in \X = \{x: \|x\|_2 \leq X_b\}$. Let the function class $\F$ be
\begin{align*}
\F=\Big\{f | f(x) =& \beta^{T}x, \beta \in \mathbb{R}^{p}, \beta^{T}A_{1}\beta \leq 1, \beta^{T}A_{2} \beta \leq 1,\\
& \sum_{j=1}^{p}c_{j\nu}\beta_{j} +\delta_{\nu} \leq 1, \delta_{\nu} > 0, \nu=1,...,V\Big\},
\end{align*} 
where $\{c_{j\nu}\}_{j,\nu}$, $\{\delta_{\nu}\}_{\nu}$, $A_1$ and $A_2$ are known beforehand. 

\edit{Let matrix $A_{\textrm{int}\gamma}$ be such that $\{ \beta: \beta^{T}A_{1}\beta \leq 1, \beta^{T}A_{2} \beta \leq 1\}$ is circumscribed by $\{\beta: \beta^{T}A_{\textrm{int}\gamma}\beta \leq 1\}$.}
Defining $\{\tilde{c}_{j\nu}\}$ and $\matXbar$ in the same way as in Section \ref{subsec:linear-bdd}, we get the following corollary.
\begin{corollary}(of Theorem \ref{theorem:polygonal-constraints}) 
\begin{equation*}
N(\sqrt{n}\epsilon,\F_{|S},\| \cdot\| _{2}) \leq 
\begin{cases}
\min\{|P^{K_{0}}|,|P_{c}^{K}|\} & \textrm{if } \epsilon < X_{b}\sqrt{\lambda_{\max}(A_{\textrm{int}\gamma}^{-1})}\\
1 & \textrm{ otherwise}
\end{cases}.
\end{equation*}
Here, $ K_{0} = \ceil*{\frac{X_{b}^{2}\lambda_{\max}(A_{\textrm{int}\gamma}^{-1})}{\epsilon^{2}}}$ and $K$ is the maximum of $K_{0}$ and 
$$
\ceil*{\frac{nX_{b}^{2}\lambda_{\max}(A_{\textrm{int}\gamma}^{-1})}{\lambda_{\min}(\matXbar\matXbar^{T})\Big[\min_{\nu=1,...,V} \frac{\delta_{\nu}}{\sum_{j=1}^{p}|\tilde{c}_{j\nu}|}\Big]^{2}}}.$$
\label{corollary:linear-quadratic-bdd}
\end{corollary}
The corollary holds for any $A_{\textrm{int}\gamma}$ that satisfies the circumscribing requirement. In particular, we can construct the ellipsoid $\{\beta: \beta^{T}A_{\textrm{int}\gamma}\beta \leq 1\}$ such that it `tightly' circumscribes the set $\{ \beta: \beta^{T}A_{1}\beta \leq 1, \beta^{T}A_{2} \beta \leq 1\}$ using Theorem \ref{theorem:kahan} in the same way as we did in Section \ref{subsec:quadratic-bdd}. \edit{The intuition for how the parameters of our side knowledge, namely, the linear inequality coefficients and the matrices corresponding to the ellipsoids, is the same as in Sections \ref{subsec:multiple-linear-bdd} and \ref{subsec:quadratic-bdd}. Relation to standard results have also been discussed in these sections.\\}

\edit{\noindent\textbf{Extension to arbitrary convex constraints:} There are at least three ways to reuse the results we have with linear, polygonal, quadratic and conic constraints to give upper bounds on covering number or empirical Rademacher complexity of function classes with arbitrary convex constraints. Such arbitrary convex constraints can arise in many settings. For instance, when the convex quadratic constraints in Section \ref{subsec:quadratic} are not symmetric around the origin, we cannot use the results of Section \ref{subsec:quadratic-bdd} directly, but the following techniques apply. Other typical convex constraints include those arising from likelihood models, entropy biases and so on. 
}

\edit{The first approach involves constructing an outer polyhedral approximation of the convex constraint set. For instance, if we are given a separation oracle for the convex constraint, constructing an outer polyhedral approximation is relatively straightforward. We can also optimize for properties like the number of facets or vertices of the polyhedron during such a construction. Given such an outer approximation, we can apply Theorem \ref{theorem:polygonal-constraints} to get an upper bound on the covering number of the hypothesis space with the given convex constraint.
}

\edit{The second approach involves constructing a circumscribing ellipsoid for the constraint set. This is possible for any convex set in general \cite{fritzjohn48}. In addition if the convex set is symmetric around the origin, the `tightness' of the circumscribing ellipsoid improves by a factor $\sqrt{p}$, where $p$ is the dimension of the linear coefficient vector $\beta$. Given such a circumscribing ellipsoid, we can apply Theorem \ref{theorem:quadratic-rad-upper-bdd} to get an upper bound on the empirical Rademacher complexity of the original function class with the convex constraint. The quality of both of these outer relaxation approaches depends on the structure and form of the convex constraint we are given. 
}

\edit{The third approach is to analyze the empirical Rademacher complexity directly using convex duality as we have done for the linear and quadratic cases, and as we will do for the conic case next.
}

\subsection{\edit{Complexity results with multiple conic constraints}}
\label{subsec:conic-bdd}

\edit{
Consider the function class $$\F = \{f: f = \beta^Tx, \beta^T\beta \leq B_b^2, \|A_k\beta\|_2 \leq a_k^T\beta + d_k\;\; \forall k=1,...,K \},$$ where we have one convex quadratic constraint and $K$ conic constraints. We can find an upper bound on the Rademacher complexity as shown below.
\begin{theorem} (Rademacher complexity of bounded linear function class with conic constraints) Let \editnew{$\X = \{x:\|x\|_2 \leq X_b\}$ with $X_b >0$ and let}
\begin{align*}
\F = \{f: f = \beta^Tx, \beta^T\beta \leq B_b^2, \|A_k\beta\|_2 \leq a_k^T\beta + d_k\;\; \forall k=1,...,K \},
\end{align*}
where \editnew{$B_b >0 $},$\{A_k,a_k,d_k\}_{k=1}^{K}$ are the parameters. Assume $A_k \succ 0$ and let $\lambda_{\min}(A_k)$ denote its minimum eigenvalue for $k=1,...,K$. Also let $\sup_{x \in \X}\|x\|_2 \leq X_b$. Then,
\begin{align*}
\Rad \leq  \frac{X_b}{\sqrt{n}}\cdot\min\left\{B_b,\sum_{k=1}^{K}\frac{B_b\|a_k\|_2 + d_k}{K\cdot\lambda_{\min}(A_k)}\right\}.
\end{align*}
\label{theorem:conic-bdd}
\end{theorem}
}

\edit{\textit{Intuition:} When $\|a_k\|_2$ and $d_k$ are $o(\lambda_{\min}(A_k))$, the effect of conic constraints can influence the upper bound on the empirical Rademacher complexity and make the corresponding generalization bounds tighter. From a geometric point of view, we can infer the following: if the cones are sharp, then $\lambda_{\min}(A_k)$ are high, implying a smaller empirical Rademacher complexity. Figure \ref{fig:conic_intuition} illustrates this in two dimensions.
}

\begin{figure}
     \centering
     \includegraphics[width=.95\textwidth]{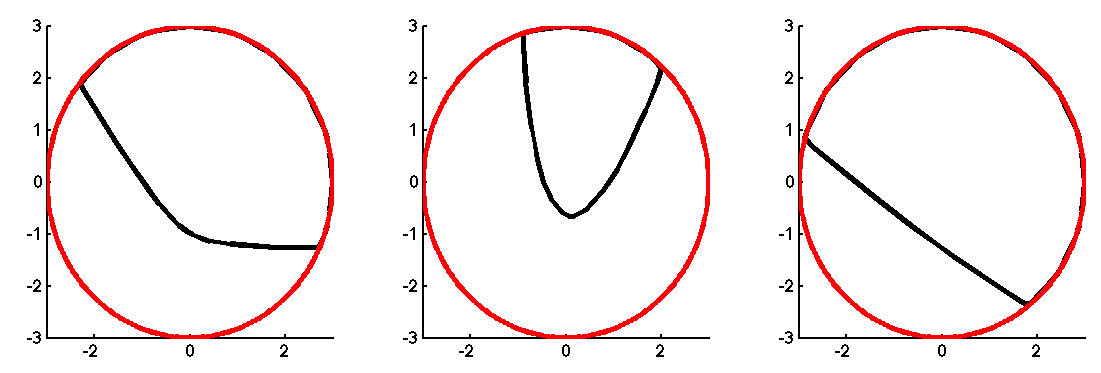}
     \caption{\edit{Here we illustrate the effect of a single conic constraint $\{\beta: \sqrt{4\mu \beta_1^2 + \mu \beta_2^2} \leq \delta(2\beta_1 + 3\beta_2 + 4)\}$ on our hypothesis space $\{\beta \in \R^2: \beta^T\beta \leq 9\}$ for different scaling values of parameters $\mu$ and $\delta$. In our notation, matrix $A = [2\sqrt{\mu} \;\;0; 0 \;\;\sqrt{\mu}]$, vector $a = \delta[2\;\;3]^T$ and scalar $d = 4\delta$. \textit{Left:} Parameter set $(\mu,\delta)$ is equal to $(1,1)$. The region covered by the conic constraint is the convex set in the upper part of the circle. \textit{Center:} Changing the parameters $(\mu,\delta)$ to $(10,1)$ makes the eigenvalue $\lambda_{\min}(A)$ larger thus reducing the intersection region further. \textit{Right:} Changing the parameters $(\mu,\delta)$ to $(1,10)$ increases the magnitude of $\|a\|_2$ and $d$ relative to the value of $\lambda_{\min}(A)$ increases the intersection region between the conic constraint and the ball. This leads to a larger empirical Rademacher complexity bound value.  }\label{fig:conic_intuition}}
\end{figure}

\edit{\textit{Relation to standard results:} The looser unconstrained version of the upper bound $\frac{X_bB_b}{\sqrt{n}}$ is recovered when there are no conic constraints or when the conic constraints are ineffective (for instance, when $\|a_k\|_2$ is high, $d_k$ is a large offset or $\lambda_{\min}(A_k)$ is small). 
}

\edit{
\begin{remark}
There have been some recent attempts to obtain bounds on a related measure, similar to the empirical Gaussian complexity defined here, in the compressed sensing literature that also involves conic constraints \cite{stojnic2009various}.
Their objective (minimum number of measurements for signal recovery assuming sparsity) is very different from our objective (function class complexity and generalization). In the former context, there are a few results \cite{chandrasekaran2012convex} dealing with the intersection of a single generic cone with a sphere ($\mathbb{S}^{p-1}$) whereas in this context, we look at the intersection of multiple second order cones (explicitly parameterized by $\{A_k,a_k,d_k\}_{k=1}^{K}$) with balls ($\{\beta^T\beta \leq B_b^2\}$).
\end{remark}
}

\section{Related Work}
\label{sec:background}
It is well-known that having additional unlabeled examples can aid in learning \cite{shental2003computing,nguyen2008improving,gomez2008semisupervised}, and this has been the subject of research in semi-supervised learning \cite{zhu05survey}. The present work is fundamentally different than semi-supervised learning, because semi-supervised learning exploits the distributional properties of the set of unlabeled examples. In this work, we do not necessarily have enough unlabeled examples to study these distributional properties, but these unlabeled examples do provide us information about the hypothesis space. Distributional properties used in semi-supervised learning include cluster assumptions \cite{singh2008unlabeled,rigollet2006generalization} and manifold assumptions \cite{belkin2004semi,belkin2004regularization}. In our work, the information we get from the unlabeled examples allows us to restrict the hypothesis space, which lets us be in the framework of empirical risk minimization and give theoretical generalization bounds via complexity measures of the restricted hypothesis spaces \cite{bartlett02,vapnik98}. While the focus of many works \citep[e.g.,][]{zhang02,maurer2006rademacher} is on complexity measures for ball-like function classes, our hypothesis spaces are more complicated, and arise here from constraints on the data.

\edit{Researchers have also attempted to incorporate domain knowledge directly into learning algorithms, where this domain knowledge does not necessarily arise from unlabeled examples. For instance, the framework of knowledge based SVMs \cite{fung2002knowledge,le2006simpler} motivates the use of various constraints or modifications in the learning procedure to incorporate specific kinds of knowledge (without using unlabeled examples). The focus of \citet{fung2002knowledge} is algorithmic and they consider linear constraints. \citet{le2006simpler} incorporate knowledge by modifying the function class itself, for instance, from linear function to non-linear functions.}

In a different framework, that of Valiant's PAC learning, there are concentration statements about the risks in the presence of unlabeled examples \cite{balcan2005pac,kaariainen2005generalization}, though in these results, the unlabeled points are used in a very different way than in our work.
Specifically, in the work of \citet{balcan2005pac}, the authors introduce the notion of incompatibility $\E_{x\sim D}[1 - \chi(h,x)]$ between a function $h$ and the input distribution $D$. 
The unlabeled examples are used to estimate the distribution dependent quantity $\E_{x\sim D}[1 - \chi(h,x)]$.
By imposing the constraint that models have their incompatibility with the distribution of the data source $D$ below a desired level, we restrict the hypothesis space. Their result for a finite hypothesis space is as follows:
\begin{theorem} \citep[Theorem 1 of][]{balcan2005pac}
If we see $m$ unlabeled examples and $n$ labeled examples, where
\begin{align*}
m \geq \frac{1}{\epsilon}\left[ \ln|C| + \ln \frac{2}{\delta}\right] \textrm{ and } n \geq \frac{1}{\epsilon}\left[ \ln|C_{D,\chi}(\epsilon)| + \ln \frac{2}{\delta}\right],
\end{align*}
then with probability $1-\delta$, all $h \in C$ with zero training error and zero incompatibility $\frac{1}{m}\sum_{i=1}^{m}(1-\chi(h,\tilde{x}_i)) = 0$, we have $\E[l(h(x),y)] \leq \epsilon$. 
\label{theorem:balcan}
\end{theorem}
Here $C$ is the finite hypothesis space of which $h$ is an element and $C_{D,\chi}(\epsilon) = \{h \in C: \E_{x\sim D}[1-\chi(h,x)] \leq \epsilon\}$. In the work of ~\citet{kaariainen2005generalization}, the author obtains a generalization bound by approximating the disagreement probability of pairs of classifiers using unlabeled data. Again, here the unlabeled data is used to estimate a distribution dependent quantity, namely, the true disagreement probability between consistent models. In particular, the disagreement between two models $h$ and $g$ is defined to be $d(h,g) = \frac{1}{m}\sum_{i=1}^{m}1_{[h(\tilde{x}_i) \neq g(\tilde{x}_i)]}$. The following theorem about generalization is proposed.
\begin{theorem}
Let $\F$ be the class of consistent models, that is, the set of models with zero training error. Assume the true model belongs to this class. Let $\hat{f} \in \F$ be the function whose distance to the farthest function in $\F$ is minimal (via metric $d$). Then, for all $S$, with probability $1-\delta$ over the choice of unlabeled sample $S^{\textrm{unlabeled}}$,
\begin{align*}
\E_{S^{\textrm{unlabeled}}}&[l(\hat{f}(x),y)] \leq \inf_{f \in \F}\sup_{g \in \F}d(f,g) \\
+& \mathcal{\bar{R}}(\{1_{[g\neq g']} | g,g' \in F\}_{|S^{\textrm{unlabeled}}}) + 
O\left(\sqrt{\frac{\ln(2/\delta)}{m}}\right).
\end{align*}
\label{theorem:kaarianen}
\end{theorem}
Note that the randomization in both Theorems \ref{theorem:balcan} and \ref{theorem:kaarianen} is also over unlabeled data. In our theorems, we do not randomize with respect to the unlabeled data. For us, they serve a different purpose and do not need to be chosen randomly. 
While their results focus on exploiting unlabeled data to estimate distribution dependent quantities, our technology focuses on exploiting unlabeled data to restrict the hypothesis space directly.

\section{Proofs}
\label{sec:proofs}

\subsection{Proof of Proposition \ref{prop:single-linear-constraint-duality}}
\begin{proof}
\edit{ 
Instead of working with the maximization problem in the definition of empirical Rademacher complexity, we will work with a couple of related maximization problems, due to the following lemma.
\begin{lemma}
\begin{align}
\Rad \leq \E\left[\max\left(\sup_{f \in \F} \frac{1}{n}\sum_{i=1}^{n}\sigma_{i}f(x_{i}),\sup_{f \in \F} -\frac{1}{n}\sum_{i=1}^{n}\sigma_{i}f(x_{i})\right)\right]. \label{eqn:radub}
\end{align}
\label{lemma:radub}
\end{lemma}
\begin{proof}
Since the empirical Rademacher complexity is defined as $\allowbreak \mathbb{E}_{\sigma}[\sup_{f \in \F} \frac{1}{n}| \sum_{i=1}^{n}\sigma_{i}f(x_{i})| ]$, we will show that for any fixed $\sigma$ vector,
\begin{align}
\sup_{f \in \F} \frac{1}{n}\left| \sum_{i=1}^{n}\sigma_{i}f(x_{i}) \right| \leq \max\left(\sup_{f \in \F} \frac{1}{n}\sum_{i=1}^{n}\sigma_{i}f(x_{i}),\sup_{f \in \F} -\frac{1}{n}\sum_{i=1}^{n}\sigma_{i}f(x_{i})\right).
\label{eqn:radubproof}
\end{align}
The inequality above is straightforward to prove. Let $f^*$ be the optimal solution to the maximization problem on the left. Then, $f^*$ is a feasible point for each of the maximization problems on the right. We will look at two cases: In the first case, let $\frac{1}{n}\sum_{i=1}^{n}\sigma_{i}f^*(x_{i}) \geq 0$. Then, clearly the first maximization problem on the right, namely, $\sup_{f \in \F} \frac{1}{n}\sum_{i=1}^{n}\sigma_{i}f(x_{i})$ will have an optimal value greater than or equal to the left side of Equation (\ref{eqn:radubproof}). In the second case, let $\frac{1}{n}\sum_{i=1}^{n}\sigma_{i}f^*(x_{i}) < 0$. Then, the second maximization problem on the right, namely, $\sup_{f\in \F} -\frac{1}{n}\sum_{i=1}^{n}\sigma_{i}f(x_{i})$ will have an optimal value greater than or equal to the left side of Equation (\ref{eqn:radubproof}). That is, in this case:
\begin{align*}
0 \leq \left|\frac{1}{n}\sum_{i=1}^{n}\sigma_{i}f^*(x_{i}) \right| = - \frac{1}{n}\sum_{i=1}^{n}\sigma_{i}f^*(x_{i}) \leq \sup_{f\in \F} -\frac{1}{n}\sum_{i=1}^{n}\sigma_{i}f(x_{i}).
\end{align*}
Combining the two cases, we get the Equation (\ref{eqn:radubproof}). Taking expectations over $\sigma$ gives us the desired inequality.
\end{proof}
}

\edit{ 
\textit{Continuing with the proof of Proposition \ref{prop:single-linear-constraint-duality}:} Let $g = \sum_{i=1}^{n}\sigma_ix_i = X_L\sigma$ so that $\Rad = \frac{1}{n}\E[\sup_{\beta \in \F} |g^T\beta|]$.  We will attempt to dualize the two maximization problems in the upper bound provided by Lemma \ref{lemma:radub} to get a bound on the empirical Rademacher complexity. Both  maximization problems are very similar except for the objective. Let $\omega(g,\F)$ be the optimal value of the following optimization problem:
\begin{align*}
\max_{\beta} g^T\beta \;\;\; \textrm{s.t.}\\
\beta^T\beta \leq B_b^2\\
a^T\beta \leq 1.
\end{align*}
Thus $\omega(g,\F)$ represents the optimal value of the maximization problem inside the expectation operation in the first term of Equation (\ref{eqn:radub}).
We will now write a dual program to the above and use weak duality to upper bound $\omega(g,\F)$. The Lagrangian is:
\begin{align*}
\mathcal{L}(\beta,\gamma,\eta) = g^T\beta + \gamma(B_b^2 - \beta^T\beta) + \eta(1 - a^T\beta),
\end{align*}
where $\beta \in \R^p, \gamma \in \R_{+}, \eta \in \R_{+}$. Maximizing the Lagrangian with respect to $\beta$ gives us:
\begin{align*}
\max_{\beta}&\;\mathcal{L}(\beta,\gamma,\eta) = \\
&= \max_{\beta}\left[(g - \eta a)^T\beta -\gamma\beta^T\beta + \gamma B_b^2 + \eta\right]\\
&= \max_{\beta}\left[ -\gamma\left[\beta^T\beta - \frac{2(g - \eta a)^T\beta}{2\gamma} + \frac{\|g - \eta a\|_2^2}{4\gamma^2}\right]  + \frac{\|g - \eta a\|_2^2}{4\gamma} + \gamma B_b^2 + \eta \right]\\
&= \max_{\beta}\left[  -\gamma\left\|\beta - \frac{g - \eta a}{2\gamma}\right\|_2^2  +  \frac{\|g - \eta a\|_2^2}{4\gamma} + \gamma B_b^2 + \eta\right]\\
&= \frac{\|g - \eta a\|_2^2}{4\gamma} + \gamma B_b^2 + \eta.
\end{align*}
The dual problem is thus
\begin{align*}
\min_{\gamma \geq 0, \eta \geq 0} \frac{\|g - \eta a\|_2^2}{4\gamma} + \gamma B_b^2 + \eta.
\end{align*}
Minimizing with respect to one of the decision variables, $\gamma$, gives the following dual problem
\begin{align*}
\min_{\eta \geq 0} B_b\|g - \eta a\|_2 + \eta.
\end{align*}
Thus, $\omega(g,\F) \leq\min_{\eta \geq 0} (B_b\|g - \eta a\|_2 + \eta)$. Similarly we can prove an upper bound on the maximization problem appearing in the second term in the max operation in Equation (\ref{eqn:radub}), which will be $\min_{\eta \geq 0} (B_b\|g + \eta a\|_2 + \eta)$. Thus, the empirical Rademacher complexity is upper bounded as:
\begin{align*}
\Rad &\\ 
\leq \frac{1}{n}&\max\left( \E\left[\min_{\eta \geq 0} (B_b\|g - \eta a\|_2 + \eta)\right],  \E\left[\min_{\eta \geq 0} (B_b\|g + \eta a\|_2 + \eta)\right]\right)\\
= \frac{1}{n}&\max\left( \E_{\sigma}\left[\min_{\eta \geq 0} (B_b\|X_L\sigma - \eta a\|_2 + \eta)\right],  \E_{\sigma}\left[\min_{\eta \geq 0} (B_b\|X_L\sigma + \eta a\|_2 + \eta)\right]\right).
\end{align*}
}
\qed
\end{proof}

\subsection{Proof of Theorem \ref{theorem:quadratic-rad-upper-bdd}}

\begin{proof}
Consider the set $\F_{|S}  = \{(\beta^{T} x_1,..., \beta^{T} x_n) \in \mathbb{R}^n : \beta^{T}\mathbb{I}\beta \leq B_{b}^{2}, \beta^{T}A_{2}\beta \leq 1 \} \subset \mathbb{R}^n$. Let $\sigma=[\sigma_{1},...,\sigma_{n}]^{T}$. 
Also, let $\alpha = A_{\textrm{int}\gamma}^{1/2}\beta$.
\begin{align*}
\Rad  
\overset{(a)}{\leq}& \frac{1}{n}\mathbb{E}_{\sigma}\Bigg[\sup_{\{\beta: \beta^{T}A_{\textrm{int}\gamma}\beta \leq 1\}} \edit{\left|\sum_{i=1}^{n} \sigma_i \beta^{T} x_i \right|} \Bigg]\\
\overset{(b)}{=}& \frac{1}{n}\mathbb{E}_{\sigma}\Bigg[\sup_{\{\alpha: \alpha^{T}\alpha \leq 1\}} \edit{\left|\sum_{i=1}^{n} \sigma_i (A_{\textrm{int}\gamma}^{-1/2}\alpha)^{T} x_i\right|} \Bigg] \\
=& \frac{1}{n}\mathbb{E}_{\sigma}\Bigg[\sup_{\{\alpha:\| \alpha\| _2 \leq 1 \}} \edit{\left|\alpha^{T}(A_{\textrm{int}\gamma}^{-1/2})^{T} \Xlab\sigma\right|} \Bigg] \\
\overset{(c)}{=}& \frac{1}{n}\mathbb{E}_{\sigma}\Bigg[ \| (A_{\textrm{int}\gamma}^{-1/2})^{T} \Xlab\sigma\| _{2} \Bigg] \\
\overset{(d)}{\leq}& \frac{1}{n}\sqrt{\mathbb{E}_{\sigma}\Bigg[ \| (A_{\textrm{int}\gamma}^{-1/2})^{T} \Xlab\sigma\| _{2}^{2} \Bigg] }\\
=& \frac{1}{n}\sqrt{\mathbb{E}_{\sigma} \Bigg[  \textrm{trace}(\Xlab^{T}A_{\textrm{int}\gamma}^{-1} \Xlab\sigma\sigma^{T} ) \Bigg] } \\
\overset{(e)}{=}& \frac{1}{n}\sqrt{ \textrm{trace}(\Xlab^{T}A_{\textrm{int}\gamma}^{-1} \Xlab) } 
\end{align*}
where $(a)$ follows because we are taking the supremum over the circumscribing ellipsoid; $(b)$ follows because $A_{\textrm{int}\gamma}$ is positive definite, hence invertible; (c) is by Cauchy-Schwarz (equality case); (d) uses Jensen's inequality and (e) uses the linearity of trace and expectation to commute them along with the fact that $\mathbb{E}[\sigma\sigma^{T}] = I$. 
\qed
\end{proof}

\subsection{Proof of Theorem \ref{theorem:quadratic-rad-lower-bdd}}

\begin{proof}
\edit{
Recall that we can decompose $\Aint$ into a product $P^T DP$ where $D$ is a diagonal matrix with the eigenvalues of $\Aint$ as its entries and $P$ is an orthogonal matrix (i.e., $P^T P=I$).}  Let us define a new variable: $\alpha :=P\beta$, which is a linear transformation of linear model parameter $\beta$. Then, the empirical Gaussian complexity of our function class \edit{can be written as:} 
\[
\Gauss =   \E_{\sigma}\left[\sup_{\alpha^{T}D\alpha\leq 1} \frac{1}{n}\sum_{i=1}^{n}\left|\sigma_i \alpha^{T}P x_i\right|\right],
\] 
where $\{\sigma_{i}\}_{i=1}^{n}$ are i.i.d. standard normal random variables. 
We now define a new vector $\omega$ to be a transformed version of the random vector $\sum_{i=1}^{n}\sigma_i  x_i$. That is, let  $\omega(\sigma) := P\sum_{i=1}^{n}\sigma_i  x_i$. We will drop the dependence of $\omega$ on $\sigma$ from the notation when it is clear from the context. The expression now becomes 
\begin{equation}\label{eqn:gauss-complexity-ineq}
\GaussScaled\; \edit{\geq}\;  \E_{\sigma}\left[\sup_{\alpha^{T}D\alpha\leq 1} \alpha^T \omega\right],
\end{equation}
\edit{where the inequality is because we removed the absolute sign in the right hand side expression before substituting for $\omega$.}

\edit{The following are the major steps in our proof:
\begin{itemize}
\item We will analyze the Gaussian function $F(\omega(\sigma)) := \sup_{\alpha^T D \alpha\leq 1} \alpha^{T}\omega(\sigma)$ and show it is Lipschitz in $\sigma$. This is proved in Lemma \ref{lemma:lipschitz-property}. 
\item Then we apply Lemma \ref{lemma:gaussian-concentration}, which is about Gaussian function concentration, to the above function. In particular, we will upper bound the variance of the Gaussian function $F(\omega(\sigma))$ in terms of its parameters (Lipschitz constant, matrix $D$, etc).
\item We then generate a candidate lower bound for the empirical Gaussian complexity. 
\item The upper bound on the variance of $F(\omega(\sigma))$ we found earlier is used to make this bound proportional to $\sqrt{\textrm{trace}(\Xlab\Aint^{-1}\Xlab)}$.
\item Finally, we use a relation between empirical Rademacher complexity and empirical Gaussian complexity to obtain the desired result.\\
\end{itemize}
}

\noindent\edit{\textbf{Computing a Lipschitz constant for $F(\omega(\sigma))$}: The following lemma gives an upper bound on the Lipschitz constant of $F(\omega(\sigma))$.} 
\begin{lemma}
The function $F(\omega(\sigma)):= \sup_{\alpha^T D \alpha\leq 1} \alpha^{T}\omega(\sigma)$ is Lipschitz in $\sigma$ with a Lipschitz constant $\mathcal{L}$ bounded by $X_b\sqrt{\frac{p\cdot n}{\lambda_{min}(D)}}$.
\label{lemma:lipschitz-property}
\end{lemma}
\begin{proof}
We have 
\begin{align*}
F(\omega)=\sup_{\alpha^T D \alpha \leq 1} \alpha^{T}\omega = \sup_{(D^{1/2}\alpha)^T (D^{1/2}\alpha) \leq 1} \alpha^{T}\omega.
\end{align*} 
Using a new dummy variable $\rho=D^{1/2}\alpha$ we have:
\begin{align*}
F(\omega)=\sup_{\rho^T \rho\leq 1} (D^{-1/2}\rho)^{T}\omega =\sup_{\rho^T \rho\leq 1} \rho^{T}(D^{-1/2})^{T}\omega=\|D^{-1/2}\omega\|_{2} .
\end{align*}
Thus,
\begin{align*}
|F(\omega_1)-F(\omega_2)| &= \left|\|D^{-1/2}\omega_{1}\|_{2} - \|D^{-1/2}\omega_{2}\|_{2}\right| \leq \|D^{-1/2}(\omega_{1} - \omega_{2})\|_{2}\\
&\overset{(a)}{\leq} \edit{\left\|\frac{1}{\sqrt{\lambda_{min}(D)}}I(\omega_{1} - \omega_{2})\right\|_2} = \frac{1}{\sqrt{\lambda_{min}(D)}} \|\omega_1-\omega_2\|_{2}.\end{align*}
At (a), we used the fact that $D^{-1} \preceq \frac{1}{\lambda_{min}(D)}I$.\\
Now, we will upper bound  $\|\omega_1-\omega_2\|_{2}$ using $\sigma_1$ and $\sigma_2$ as follows. Using the definition of $\omega = P\Xlab\sigma$ we get,
\begin{align*}
\|\omega_1-\omega_2\|_{2} &= \|P\Xlab\sigma_1 - P\Xlab\sigma_2\|_{2} = \|P\Xlab(\sigma_1 - \sigma_2)\|_{2}\\
&\stackrel{(b)}{\leq} \|\Xlab(\sigma_1 - \sigma_2)\|_{2}\\
&\stackrel{}{=} \sqrt{(\sigma_1 - \sigma_2)^T\Xlab^T\Xlab(\sigma_1 - \sigma_2)}\\
&\stackrel{(c)}{\leq} \sqrt{(\sigma_1 - \sigma_2)^T\lambda_{max}(\Xlab^T\Xlab)I(\sigma_1 - \sigma_2)}\\
&\stackrel{}{=} \sqrt{ \lambda_{max}(\Xlab^T \Xlab) } \|\sigma_1 - \sigma_2\|_2\\
&\stackrel{(d)}{\leq} X_{b} \sqrt{ p\cdot n}\|(\sigma_1 - \sigma_2)\|_{2}.
\end{align*} 
Here, (b) follows because $P$ is an orthonormal matrix, (c) because $  \Xlab^T \Xlab \preceq \lambda_{max}(\Xlab^T \Xlab)I $ and (d) because $ \lambda_{max}(\Xlab^T \Xlab) \leq \textrm{trace}(\Xlab^T \Xlab) = \sum_{i=1}^{n}(\Xlab^T \Xlab)_{ii}$. Since, each diagonal element of $\Xlab^T \Xlab$ is a sum of $p$ terms each upper bounded by $X_b^2$, we have $ \lambda_{max}(\Xlab^T \Xlab) \leq n\cdot p \cdot X_b^2$.
\qed
\end{proof}

\noindent\edit{\textbf{Upper bounding the variance of $F(\omega(\sigma))$ using Gaussian concentration}:} The following lemma describes concentration for Lipschitz functions of gaussian random variables.
\begin{lemma} \citep[Concentration,][]{tsirel1976norms} If $\sigma$ is a vector with i.i.d. standard normal entries and $G$ is any function with Lipschitz constant $\mathcal{L}$ (with respect to the Euclidean norm), then 
\begin{align*}
\Pr[\left|(G(\sigma)-\E[G(\sigma)]\right| \geq t] \leq 2 e^{-\frac{t^2}{2\mathcal{L}^{2}}}.
\end{align*} 
\label{lemma:gaussian-concentration}
\end{lemma}

\edit{The proof of Lemma \ref{lemma:gaussian-concentration} is omitted here. Using Lemmas \ref{lemma:lipschitz-property} and \ref{lemma:gaussian-concentration}} with $G(\sigma) = F(\omega)$, we have
\begin{align}
\Pr[\left|(F(\omega)-\E_{\sigma}[F(\omega)]\right| \geq t] \leq 2 e^{-\frac{t^2}{2\mathcal{L}^2}}, \label{eqn:lipschitz-substitution}
\end{align} 
where $\mathcal{L} = X_b\sqrt{\frac{p\cdot n}{\lambda_{min}(D)}}$.

Let $Y=|(F(\omega)-\E_{\sigma}[F(\omega)]|$. Then from the above tail bound, $P(Y^{2} \geq s) \leq 2 e^{-\frac{s}{2\mathcal{L}^2}}$ is true.  Now we can bound the variance of $F(\omega)$ \edit{using the above inequality and the following lemma}.
\begin{lemma} 
For a random variable $Y^2$, $\E[Y^2]=\int^{+\infty}_{0}P(Y^2\geq s)ds$. 
\label{lemma:squared-rv}
\end{lemma}
\begin{proof}This is an alternate expression for the expectation of a non-negative univariate random variable in terms of its distribution function. To show this, let us assume that the density function of $Y^2$ is $\mu_{Y^2}$. We then have
$P(Y^2 \geq s)=1-P(Y^2\leq s)=1-\int^{s}_{0}\mu_{Y^2}(s')ds'$ and thus: $\mu_{Y^2}(s)=-\frac{dP(Y^2 \geq s)}{ds}.$ 
So,
\begin{align*}
\E[Y^2]&=\int^{+\infty}_{0}s \mu_{Y^2}(s)ds=-\int^{+\infty}_{0}s \frac{dP(Y^2\geq s)}{ds}ds\\
&= -[sP(Y^2 \geq s)]^{+\infty}_{0}+\int^{+\infty}_{0}P(Y^2\geq s)ds.
\end{align*}
The first term is zero and we obtain our expression.
\qed
\end{proof} 
 
The variance of $F(\omega)$, which is the same as the expectation of $Y^{2}$, can thus be \edit{upper bounded} as \edit{follows}:
\begin{align}\label{varF}
\textrm{Var}(F(\omega))&=\E_{\sigma}(Y^2)\overset{(a)}=\int^{+\infty}_{0}P(Y^2\geq s)ds \nonumber\\
\overset{(b)}\leq& 2 \int^{+\infty}_{0}e^{-\frac{s}{2\mathcal{L}^2}}ds=4X_b^2{\frac{p\cdot n}{\lambda_{min}(D)}},
\end{align}
where we \edit{used Lemma \ref{lemma:squared-rv} for step (a) and Equation (\ref{eqn:lipschitz-substitution}) for step (b) and finally substituting $X_b\sqrt{\frac{p\cdot n}{\lambda_{min}(D)}}$ for $\mathcal{L}$.} \\

\noindent\edit{\textbf{Lower bounding the empirical Gaussian complexity}:}
\edit{Now} we will lower bound the \edit{empirical} Gaussian complexity by constructing a feasible candidate $\alpha'$ to substitute for the $\sup$ operation in Equation (\ref{eqn:gauss-complexity-ineq}). \edit{Later,} we will use the variance upper bound on $F(\omega)$ \edit{we found in the earlier section to make the bound more specific}.

Let $j^{*} \in \{1,...,p\}$ be the index at which the diagonal element $D(j^{*},j^{*}) = \lambda_{min}(D)$. 
For each realization of $\sigma$ (or equivalently $\omega$) let $\alpha' = \left[0\hdots \frac{|\omega_{j^{*}}|}{\omega_{j^{*}}\sqrt{\lambda_{min}(D)}}\hdots 0\right]$ with the non-zero entry at coordinate $j^{*}$. Clearly $\alpha'$ is a feasible vector in the ellipsoidal constraint $\{\alpha: \alpha^{T}D\alpha \leq 1\}$ seen in the complexity expression, Equation (\ref{eqn:gauss-complexity-ineq}).
Substituting it and using the definition of $F(\omega)$, we get a lower bound on the \edit{empirical Gaussian} complexity:
\begin{align*}
\GaussScaled \edit{\geq}& \E_{\sigma}[F(\omega)]  = \E_{\sigma}\left[\sup_{\alpha^{T}D\alpha \leq 1}\alpha^{T}\omega\right] \\
\overset{(a)}\geq& \E_{\sigma}[(\alpha')^{T}\omega] \overset{(b)}\geq \frac{1}{\sqrt{\lambda_{min}(D)}}\E_{\sigma}[|\omega_{j^{*}}|].
\end{align*}
Step (a) comes from the fact that $\alpha'$ is feasible in $\{\alpha: \alpha^{T}D\alpha \leq 1\}$ but not necessarily the maximum, and step (b) comes from the definition of $\alpha'$.\\

\noindent\edit{\textbf{Making the lower bound more specific using variance of $F(\omega(\sigma))$:}}
Note that compared to the upper bound on the related Rademacher complexity obtained in Theorem \ref{theorem:quadratic-rad-upper-bdd}, the dependence \edit{of empirical Gaussian complexity} on $A_{\textrm{int}\gamma}$ is weak (only via $\lambda_{min}(D)$). We will use the variance of $F(\omega)$ to obtain a lower bound very similar to the upper bound in Equation (\ref{eqn:quadratic-rad-upper-bdd}).
Rearranging the terms in the previous inequality, we get:
\begin{align}
\frac{(\E_{\sigma}[F(\omega)])^2}{ (\E_{\sigma}|\omega_{j^{*}}|)^2} \geq \frac{1}{\lambda_{min}(D)}.
\label{eqn:rad-loose1}
\end{align}

By rewriting the variance in terms of the second and first moments, using expression (\ref{varF}) and then using (\ref{eqn:rad-loose1}) we get
\begin{align*}
\textrm{Var}(F(\omega)) =& \E_{\sigma}[F^{2}(\omega)]-(\E_{\sigma}[F(\omega)])^2\\
           \leq& 4X_b^2{\frac{p\cdot n}{\lambda_{min}(D)}} \leq 4p n X_b^2\frac{(\E_{\sigma}[F(\omega)])^2}{(\E_{\sigma}|\omega_{j^{*}}|)^2}.
\end{align*}
Using expression (\ref{eqn:gauss-complexity-ineq}) again, and then rearranging the terms in the previous expression, we obtain another lower bound on the scaled Gaussian complexity, which is:
\begin{align}
\left(\GaussScaled\right)^{2} \edit{\geq}& (\E_{\sigma}[F(\omega)])^{2} 
 \geq \frac{\E_{\sigma}[(F(\omega))^2]}{1+\frac{4pnX_b^2}{(\E_{\sigma}|\omega_{j^{*}}|)^2}}\nonumber\\
=& \frac{\E_{\sigma}[(\sup_{\alpha^{T}D\alpha \leq 1}\omega^T \alpha)^2]}{1+\frac{4pnX_b^2}{(\E_{\sigma}|\omega_{j^{*}}|)^2}}.
\label{eqn:rad-tighter}
\end{align}

We can now try to bound two easier quantities $\E_{\sigma}[(\sup_{\alpha^{T}D\alpha \leq 1}\omega^T \alpha)^2]$ and $\E_{\sigma}|\omega_{j^{*}}|$ to get an expression for scaled Gaussian complexity and consequently for \edit{the empirical} Rademacher complexity.

Let us start first with $\E|\omega_{j^{*}}|$. By definition $\omega $ equals $P\Xlab\sigma$. Thus, the $j^{*}$th coordinate of $\omega$ will be $\sum_{i}\sigma_{i}(Px_{i})_{j^{*}}$ where $(\cdot)_{j^*}$ represents the $j^{*}$th coordinate of the vector. Since the $\sigma_{i}$ are independent standard normal, their weighted sum $\omega$ is also standard normal with variance $\sum_{i}(Px_{i})_{j^{*}}^{2}$. Since for any normal random variable $z$ with mean zero and variance $d$ it is true that $\E[|z|] = \sqrt{\frac{2d}{\pi}}$, we have
\begin{align}\label{firstinbound}
\E_{\sigma}[|w_{j^{*}}|] =& \sqrt{\frac{2}{\pi}}\left(\sum_{i}(Px_{i})_{j^{*}}^{2}\right)^{\frac{1}{2}}\nonumber \\
\geq& \sqrt{\frac{2}{\pi}}\min_{j=1,...,p}\|(P\Xlab)_{j}\|_{2}
\end{align}
where $(P\Xlab)_{j}$ represents the $j^{th}$ row of the matrix $P\Xlab$.
For the second moment term of (\ref{eqn:rad-tighter}) that we need to bound, $\E_{\sigma}[(\sup_{\alpha^{T}D\alpha \leq 1}\omega^T \alpha)^2]$, we can see that 
\begin{align*}
\sup_{\alpha^{T}D\alpha \leq 1}\omega^T \alpha =& \sup_{\tilde{\alpha}^{T}\tilde{\alpha} \leq 1}(P\Xlab\sigma)^{T}D^{-1/2}\tilde{\alpha} \\
=& \|D^{-1/2}P\Xlab\sigma\|_{2}.
\end{align*}
Thus,
\begin{align}\nonumber
\E_{\sigma}\Big[\Big(\sup_{\alpha^{T}D\alpha \leq 1}\omega^T \alpha\Big)^2 \Big] =& \E_{\sigma}[\|D^{-1/2}P\Xlab\sigma\|_{2}^{2}] \nonumber\\
&= \E_{\sigma}[(D^{-1/2}P\Xlab\sigma)^{T}D^{-1/2}P\Xlab\sigma]\nonumber \\
&= \E_{\sigma}[\sigma^{T}\Xlab^{T} A_{\textrm{int}\gamma}^{-1}\Xlab\sigma] \nonumber\\
&= \E_{\sigma}[ \textrm{trace}(\sigma^{T}\Xlab^{T} A_{\textrm{int}\gamma}^{-1}\Xlab\sigma)] \nonumber\\
&= \E_{\sigma}[ \textrm{trace}(\Xlab^{T}A_{\textrm{int}\gamma}^{-1} \Xlab\sigma\sigma^{T} )] \nonumber\\
&= \textrm{trace}(\Xlab^{T}A_{\textrm{int}\gamma}^{-1} \Xlab).\label{secondinbound} 
\end{align}

Substituting the two bounds we just derived, (\ref{firstinbound}) and (\ref{secondinbound}), into (\ref{eqn:rad-tighter}) gives us a lower bound on the scaled Gaussian complexity:
\begin{align*}
\left(\GaussScaled\right)^{2} &\geq \frac{\textrm{trace}(\Xlab^{T}A_{\textrm{int}\gamma}^{-1} \Xlab)}{ 1 + \frac{4pnX_b^2}{(\sqrt{\frac{2}{\pi}}\min_{j=1,...,p}\|(P\Xlab)_{j}\|_{2})^{2}}}\\
\GaussScaled &\geq \sqrt{\frac{\textrm{trace}(\Xlab^{T}A_{\textrm{int}\gamma}^{-1} \Xlab)}{ 1 + \frac{4pnX_b^2}{(\sqrt{\frac{2}{\pi}}\min_{j=1,...,p}\|(P\Xlab)_{j}\|_{2})^{2}}}}.
\end{align*}

\noindent\edit{\textbf{Using the relation between Rademacher and Gaussian complexities:}}
The empirical Gaussian complexity is related to the empirical Rademacher complexity as follows. 
\begin{lemma}\citep[Lemma 4 of][]{bartlett02} 
There are absolute constants $C$ and $C'$ such that for every $\F_{|S}$ with $|S| = n$,
\begin{align*}
C'\Rad \leq \Gauss \leq C\log(n) \Rad.
\end{align*}
\label{lemma:gaussian-rademacher}
\end{lemma} 

Using \edit{the above result} gives:
\begin{align*}
{n}C\log(n) \Rad &\geq  \sqrt{\frac{\textrm{trace}(\Xlab^{T}A_{\textrm{int}\gamma}^{-1} \Xlab)}{ 1 + \frac{4pnX_b^2}{(\sqrt{\frac{2}{\pi}}\min_{j=1,...,p}\|(P\Xlab)_{j}\|_{2})^{2}}}}
\end{align*}
\edit{Thus, we get our desired result:}
\begin{align*}
&\Rad \geq \frac{\kappa}{n\log n}\sqrt{\textrm{trace}(\Xlab^{T}A_{\textrm{int}\gamma}^{-1} \Xlab)},\\
\textrm{where}&\\
&\kappa = \frac{\edit{1}}{C\sqrt{1 + \frac{2\pi pnX_b^2}{(\min_{j=1,...,p}\|(P\Xlab)_{j}\|_{2})^{2}}}}.
\end{align*}
\qed
\end{proof}

\subsection{Proof of Corollary \ref{corollary:linear-quadratic-bdd}}

\begin{proof}
Since the ellipsoid defined using $A_{\textrm{int}\gamma}$ circumscribes the region of intersection of ellipsoids determined by $A_1$ and $A_2$, we have
\begin{align*}
\F=\Big\{f | f(x) =& \beta^{T}x, \beta \in \mathbb{R}^{p}, \beta^{T}A_{1}\beta \leq 1, \beta^{T}A_{2} \beta \leq 1,\\
& \sum_{j=1}^{p}c_{j\nu}\beta_{j} +\delta_{\nu} \leq 1, \delta_{\nu} > 0, \nu=1,...,V\Big\}\\
\subseteq\\
\Big\{f | f(x) =& \beta^{T}x, \beta \in \mathbb{R}^{p}, \beta^{T}A_{\textrm{int}\gamma}\beta \leq 1,\\
& \sum_{j=1}^{p}c_{j\nu}\beta_{j} +\delta_{\nu} \leq 1, \delta_{\nu} > 0, \nu=1,...,V\Big\} =: \F'.
\end{align*}
Further,  $\beta^{T}\lambda_{\min}(A_{\textrm{int}\gamma})I\beta \leq \beta^{T}A_{\textrm{int}\gamma}\beta \leq 1$ since $\lambda_{\min}(A_{\textrm{int}\gamma})I \preceq A_{\textrm{int}\gamma}$. That is, the set $\beta^{T}\lambda_{\min}(A_{\textrm{int}\gamma})I\beta \leq 1$ is bigger than the ellipsoid defined using $  A_{\textrm{int}\gamma}$. Thus,
\begin{align*}
\F'=\Big\{f | f(x) =& \beta^{T}x, \beta \in \mathbb{R}^{p}, \beta^{T}A_{\textrm{int}\gamma}\beta \leq 1,\\
& \sum_{j=1}^{p}c_{j\nu}\beta_{j} +\delta_{\nu} \leq 1, \delta_{\nu} > 0, \nu=1,...,V\Big\}\\
\subseteq\\
\Big\{f | f(x) =& \beta^{T}x, \beta \in \mathbb{R}^{p}, \beta^{T}\beta \leq \frac{1}{\lambda_{\min}(A_{\textrm{int}\gamma})},\\
& \sum_{j=1}^{p}c_{j\nu}\beta_{j} +\delta_{\nu} \leq 1, \delta_{\nu} > 0, \nu=1,...,V\Big\} =: \F''.
\end{align*}
Noting that $\beta^{T}\beta \leq \frac{1}{\lambda_{\min}(A_{\textrm{int}\gamma})}$ is the same as $\|\beta\|_2 \leq \sqrt{\lambda_{\max}(A_{\textrm{int}\gamma}^{-1})}$, we can use Theorem \ref{theorem:polygonal-constraints} on $\F''$ with $r=2,q=2$ and $\B_b := \sqrt{\lambda_{\max}(A_{\textrm{int}\gamma}^{-1})}$ to get a bound on $N(\sqrt{n}\epsilon,\F''_{|S},\| \cdot\| _{2}) \geq N(\sqrt{n}\epsilon,\F_{|S},\| \cdot\| _{2})$ giving us the stated result. 
\qed
\end{proof}

\subsection{Proof of Theorem \ref{theorem:quadratic-rad-duality}}

\begin{proof}
\edit{
Let $g = \sum_{i=1}^{n}\sigma_ix_i = X_L\sigma$ so that $\Rad = \frac{1}{n}\E[\sup_{\beta \in \F} |g^T\beta|]$.  Instead of directly working with the empirical Rademacher complexity, we will dualize the two maximization problems in the upper bound given by Equation (\ref{eqn:radub}) of Lemma \ref{lemma:radub}. Both maximization problems are very similar except for the objective. Let $\omega(g,\F)$ be the optimal value of the following optimization problem:
\begin{align*}
\max_{\beta} g^T\beta \;\;\; \textrm{s.t.}\\
\beta^T\beta \leq B_b^2\\
\beta^TA_2\beta  \leq 1.
\end{align*}
Thus $\omega(g,\F)$ is proportional to the first term inside the max operation in Equation (\ref{eqn:radub}), which gives an upper bound in the empirical Rademacher complexity.
We will now write a dual program to the above and use weak duality to upper bound $\omega(g,\F)$. The Lagrangian is:
\begin{align*}
\mathcal{L}(\beta,\gamma,\eta) = g^T\beta + \gamma(B_b^2 - \beta^T\beta) + \eta(1 - \beta^TA_2\beta),
\end{align*}
where $\beta \in \R^p, \gamma \in \R_{+}, \eta \in \R_{+}$. Maximizing the Lagrangian with respect to $\beta$ gives us:
\begin{align*}
\max_{\beta}&\;\mathcal{L}(\beta,\gamma,\eta) = \\
&= \max_{\beta}\left[g^T\beta -\gamma\beta^T\beta -\eta\beta^TA_2\beta + \gamma B_b^2 + \eta\right]\\
&= \max_{\beta}\left[ -\left(-g^T\beta +\beta^T(\gamma\mathbb{I} +\eta A_2)\beta\right) + \gamma B_b^2 + \eta \right]\\
&= \max_{\beta}\left[ -\left(-g^T(\gamma\mathbb{I} +\eta A_2)^{-1/2}(\gamma\mathbb{I} +\eta A_2)^{1/2}\beta \right.\right.\\
&\;\;\;\;\;\;\;\;\left.\left. +\beta^T(\gamma\mathbb{I} +\eta A_2)^{1/2}(\gamma\mathbb{I} +\eta A_2)^{1/2}\beta\right) + \gamma B_b^2 + \eta \right]\\
&= \max_{\beta}\left[  -\left\|(\gamma\mathbb{I} +\eta A_2)^{1/2}\beta - \frac{(\gamma\mathbb{I} +\eta A_2)^{-1/2}g}{2}\right\|_2^2  \right.\\
&\;\;\;\;\;\;\;\;\left. +  \frac{\|(\gamma\mathbb{I} +\eta A_2)^{-1/2}g\|_2^2}{4} + \gamma B_b^2 + \eta\right]\\
&= \frac{\|(\gamma\mathbb{I} +\eta A_2)^{-1/2}g\|_2^2}{4} + \gamma B_b^2 + \eta,
\end{align*}
where in the last step we set $\beta = \frac{(\gamma\mathbb{I} +\eta A_2)^{-1}g}{2}$. The dual problem is thus:
\begin{align*}
\min_{\gamma \geq 0, \eta \geq 0} \frac{\|(\gamma\mathbb{I} +\eta A_2)^{-1/2}g\|_2^2}{4} + \gamma B_b^2 + \eta& \textrm{, or equivalently,}\\
\min_{\gamma \geq 0, \eta \geq 0} \frac{1}{4}g^T(\gamma\mathbb{I} +\eta A_2)^{-1}g + \gamma B_b^2 + \eta.&\\
\end{align*}
If we let $\gamma = 1-\eta$, we are further constraining the minimization problem, yielding another upper bound of the form:
\begin{align*}
\omega(g,\F) \leq \min_{\eta \in [0,1]} \frac{1}{4}g^T(\mathbb{I} +\eta (A_2 -\mathbb{I}))^{-1}g +  B_b^2 + \eta(1-B_b^2).
\end{align*}
If we consider the second maximization problem $\sup_{\beta \in \F} -g^T\beta$ that appears in Equation (\ref{eqn:radub}), we can similarly upper bound its optimal value with the same minimization problem as $\omega(g,\F)$. One intuitive reason why the same minimization problem serves as an upper bound is because the hypothesis class $\F$ is closed under negation. Thus, we get an upper bound on the empirical Rademacher complexity as:
\begin{align*}
\Rad &\leq \E\left[\frac{1}{n}\omega(g,\F)\right]\\
&\leq \E\left[\frac{1}{n}\min_{\eta \in [0,1]} \frac{1}{4}g^T(\mathbb{I} +\eta (A_2 -\mathbb{I}))^{-1}g +  B_b^2 + \eta(1-B_b^2)\right],
\end{align*}
where recall that $g = \sum_{i=1}^{n}\sigma_ix_i$. Fix any feasible $\eta$. Let $A_{\textrm{int}\eta} := (\mathbb{I} +\eta (A_2 -\mathbb{I}))$ (it corresponds to an ellipsoid as well since $\eta \in [0,1]$).  Then,
\begin{align*}
\Rad &\leq \E\left[\frac{1}{4n} \sigma^T X_L^TA_{\textrm{int}\eta}^{-1}X_L\sigma +  \frac{1}{n}(B_b^2 + \eta(1-B_b^2))\right]\\
&= \frac{1}{4n} \textrm{trace}(X_L^TA_{\textrm{int}\eta}^{-1}X_L) +  \frac{1}{n}(B_b^2 + \eta(1-B_b^2)).
\end{align*}
We can minimize the right hand side over $\eta \in [0,1]$ to get the desired result.
}
\qed
\end{proof}

\subsection{Proof of Theorem \ref{theorem:conic-bdd}}

\begin{proof}
\edit{
The core idea of the proof is to come up with an intuitive upper bound on the empirical Rademacher complexity of $\F$ using convex duality. We have already seen the use of convex duality in Proposition \ref{prop:single-linear-constraint-duality} and Theorem \ref{theorem:quadratic-rad-duality}. Recall the definition of the empirical Rademacher complexity of a function class $\F$:
\begin{align*}
\Rad = \frac{1}{n}\E_{\sigma}\left[\sup_{\beta \in \F} \left|\sum_{i=1}^{n}\sigma_i(\beta^Tx_i)\right| \right],
\end{align*}
where $\{\sigma_{i}\}_{i=1}^{n}$ are i.i.d. Bernoulli random variables taking values in $\{\pm 1\}$ with equal probability. Now define a new vector $g$ to be the random vector $\sum_{i=1}^{n}\sigma_i  x_i$. As in the previous proofs, instead of directly working with the empirical Rademacher complexity, we will dualize the two maximization problems in the upper bound given by Equation (\ref{eqn:radub}) of Lemma \ref{lemma:radub}. Let $\omega(g,\F) = \sup_{\beta \in \F}g^T\beta$. That is, $\omega(g,\F)$ is the optimal value of the first maximization problem (ignoring factor $1/n$) appearing on the right hand side of Equation (\ref{eqn:radub}):
\begin{align}
\max_{\beta}\;\;& g^T\beta \;\;\; \textrm{    s.t. }\nonumber\\
& \beta^T\beta  \leq B_b^2\nonumber\\
&  \|A_k\beta\|_2 \leq a_k^T\beta + d_k\;\; \forall k=1,...,K. \label{eqn:maxomega}
\end{align}
The Lagrangian of the problem can be written as \cite{boyd2004convex}:
\begin{align*}
\mathcal{L}(\beta,\gamma,\{z_k,\theta_k\}_{k=1}^{K}) = g^T\beta + \gamma(B_b^2 - \beta^T\beta) + \sum_{k=1}^{K}\Big[z_k^TA_k\beta + \theta_k\cdot( a_k^T\beta + d_k)\Big],
\end{align*}
where $\beta \in \R^p, \gamma \in \R_{+}$ and for $k=1,...,K$ we have $\|z_k\|_2 \leq \theta_k$. For any set of feasible values of $(\beta,\gamma,\{z_k,\theta_k\}_{k=1}^{K})$, the objective of the SOCP in Equation (\ref{eqn:maxomega})  is upper bounded by $\mathcal{L}(\beta,\gamma,\{z_k,\theta_k\}_{k=1}^{K})$. Thus, $\omega(g,\F) \leq \sup_{\beta}\mathcal{L}(\beta,\gamma,\{z_k,\theta_k\}_{k=1}^{K})$. We will analyze this maximization problem as the first step towards a tractable bound on $\omega(g,\F)$.
}

\edit{
In the second step, we will  minimize $ \sup_{\beta}\mathcal{L}(\beta,\gamma,\{z_k,\theta_k\}_{k=1}^{K})$ over variable $\gamma$ (one of the dual variables) to get an upper bound on $\omega(g,\F)$ in terms of $\{z_k,\theta_k\}_{k=1}^{K}$. These two steps are shown below:
}

\noindent \edit{\textbf{First step:} After rearranging terms and completing squares, we get the following dual objective to be minimized over dual variables $\gamma$ and $\{z_k,\theta_k\}_{k=1}^{K}$.
\begin{align*}
\sup_{\beta \in \R^p}\mathcal{L}&(\beta,\gamma,\{z_k,\theta_k\}_{k=1}^{K}) \\
& = \sup_{\beta \in \R^p} \left[\left(g + \sum_{k=1}^{K}(A_k^Tz_k + \theta_ka_k)\right)^T\beta + \gamma B_b^2 + \sum_{k=1}^{K}\theta_kd_k - \gamma\beta^T\beta\right]\\
& =  \sup_{\beta \in \R^p}\left[-\gamma\left\|\beta - \frac{g +\sum_{k=1}^{K}(A_k^Tz_k + \theta_ka_k)}{2\gamma}\right\|_2^2 \right.\\
&\;\;\;\;\;\;\;\; \left. + \frac{\|g + \sum_{k=1}^{K}(A_k^Tz_k + \theta_ka_k)\|_2^2}{4\gamma} + \left(\gamma B_b^2 + \sum_{k=1}^{K}\theta_kd_k\right) \right]\\
& =  \frac{\|g + \sum_{k=1}^{K}(A_k^Tz_k + \theta_ka_k)\|_2^2}{4\gamma} + \gamma B_b^2 + \sum_{k=1}^{K}\theta_kd_k.
\end{align*}
The second to last equality above is obtained by completing the squares (in terms of $\beta$) and the last equality is due to the fact that the optimal value is obtained when $\beta = \frac{g + \sum_{k=1}^{K}(A_k^Tz_k + \theta_ka_k)}{2\gamma}$. The resulting term is now a function of the remaining variables ($\gamma$ and $\{z_k,\theta_k\}_{k=1}^{K}$) and serves as an upper bound to $\omega(g,\F)$ for any feasible values of $\gamma$ and $\{z_k,\theta_k\}_{k=1}^{K}$.
}

\noindent \edit{\textbf{Second step:}  Since $\min_{x,y}f(x,y) = \min_x(\min_y f(x,y))$ when $f(x,y)$ is convex and the feasible set is convex, we now minimize with respect to $\gamma$ to get the following upper bound:
\begin{align*}
\inf_{\gamma \in \R_+}\sup_{\beta \in \R^p}\mathcal{L}&(\beta,\gamma,\{z_k,\theta_k\}_{k=1}^{K})\\
& = B_b\left\|g +  \sum_{k=1}^{K}(A_k^Tz_k + \theta_ka_k)\right\|_2 + \sum_{k=1}^{K}\theta_kd_k,
\end{align*}
where the above statement follows because for a problem of the form $\min_{\gamma  \in \R_+} \frac{a}{\gamma} + b\gamma +c$ with $a>0, b>0$, the optimal solution is $\gamma^* = +\sqrt{\frac{a}{b}}$.
}

\edit{
Continuing, we now optimize over the remaining variables $\{z_k,\theta_k\}_{k=1}^{K}$ as follows:
\begin{align}
\omega(g,\F) &= \sup_{\beta \in \F}g^T\beta \nonumber\\
&\leq \inf_{\{(z_k,\theta_k): \|z_k\|_2 \leq \theta_k, k=1,..,K\}}  B_b\left\|g +  \sum_{k=1}^{K}(A_k^Tz_k + \theta_ka_k)\right\|_2 + \sum_{k=1}^{K}\theta_kd_k.
\label{eqn:conic-dual-program-minimize}
\end{align}
An upper bound on $\omega(g,\F)$ can be obtained by finding a set of optimal or feasible values for $\{z_k,\theta_k\}_{k=1}^{K}$. Note that since $A_k \succ 0$, $A_k^T = A_k$ and $A_k^{-1}$ exists. Obtaining the optimal value of the minimization in Equation (\ref{eqn:conic-dual-program-minimize}) is difficult analytically. Instead, we will pick a suitable feasible value for $\{z_k,\theta_k\}_{k=1}^{K}$. Plugging this feasible value will give us an upper bound on $\omega(g,\F)$. In particular, let $z_k = - \frac{1}{K}A_k^{-1}g$. Then, setting $\theta_k = \frac{1}{K}\|A_k^{-1}g\|_2$ gives us a feasible value for each $\{z_k,\theta_k\}$. Thus,
\begin{align*}
\omega(g,\F) &\leq  B_b\left\|g +  \sum_{k=1}^{K}A_k^T\left(-\frac{1}{K}A_k^{-1}g\right) + \sum_{k=1}^{K}\frac{1}{K}\|A_k^{-1}g\|_2a_k\right\|_2 + \sum_{k=1}^{K} \frac{1}{K}\|A_k^{-1}g\|_2d_k\\
& = B_b\left\|g -g  + \sum_{k=1}^{K}\frac{\|A_k^{-1}g\|_2}{K}a_k\right\|_2 + \sum_{k=1}^{K}\frac{\|A_k^{-1}g\|_2}{K}d_k\\
& = B_b\left\|\sum_{k=1}^{K}\frac{\|A_k^{-1}g\|_2}{K}a_k\right\|_2 + \sum_{k=1}^{K}\frac{\|A_k^{-1}g\|_2}{K}d_k\\
& \leq  \sum_{k=1}^{K}\frac{\|A_k^{-1}g\|_2}{K}(B_b\|a_k\|_2 + d_k)\\
& \leq \|g\|_2\sum_{k=1}^{K}\frac{B_b\|a_k\|_2 + d_k}{K\cdot\lambda_{\min}(A_k)}.
\end{align*}
}

\edit{
Dualizing the second maximization problem in Equation (\ref{eqn:radub}) also gives us the same upper bound as obtained above for $\omega(g,\F)$. That is, if $\omega'(g,\F) := \sup_{\beta \in \F} -g^T\beta$, then the same analysis as above (replacing $g$ with $-g$) gives:
\begin{align*}
\omega'(g,\F) \leq  \|g\|_2\sum_{k=1}^{K}\frac{B_b\|a_k\|_2 + d_k}{K\cdot\lambda_{\min}(A_k)}.
\end{align*}
}

\edit{
We can now come up with the desired upper bound for the empirical Rademacher complexity using Equation (\ref{eqn:radub}):
\begin{align*}
\Rad &\leq \E\left[\max\left(\frac{1}{n}\omega(g,\F),\frac{1}{n}\omega'(g,\F)\right)\right]\\
&\leq \frac{1}{n}\E\left[ \|g\|_2\sum_{k=1}^{K}\frac{B_b\|a_k\|_2 + d_k}{K\cdot\lambda_{\min}(A_k)} \right] \;\;\; \textrm{(since upper bounds are the same)}\\
& = \frac{1}{n} \E_{\sigma}\left[ \Big\|\sum_{i=1}^{n}\sigma_i x_i\Big\|_2\right] \sum_{k=1}^{K}\frac{B_b\|a_k\|_2 + d_k}{K\cdot\lambda_{\min}(A_k)}\\
& \leq \frac{1}{n} \sqrt{\E_{\sigma}\Big[ \Big\|\sum_{i=1}^{n}\sigma_i x_i\Big\|_2^2}\Big] \sum_{k=1}^{K}\frac{B_b\|a_k\|_2 + d_k}{K\cdot\lambda_{\min}(A_k)} \;\;\;\textrm{ (by Jensen's inequality)}\\
&\leq \frac{X_b}{\sqrt{n}}\sum_{k=1}^{K}\frac{B_b\|a_k\|_2 + d_k}{K\cdot\lambda_{\min}(A_k)}.
\end{align*}
In the case when there are no active conic constraints, we cannot use this bound. Instead, we can recover the well known standard bound by removing the terms related to conic constraints in Equation (\ref{eqn:conic-dual-program-minimize}) and obtain only $\frac{X_bB_b}{\sqrt{n}}$. Combining both bounds we get,
\begin{align*}
\Rad \leq  \frac{X_b}{\sqrt{n}}\cdot\min\left\{B_b,\sum_{k=1}^{K}\frac{B_b\|a_k\|_2 + d_k}{K\cdot\lambda_{\min}(A_k)}\right\}.
\end{align*}
} 
\qed
\end{proof}

\section{{Conclusion}}\label{sec:conclusion}

In this paper, we have outlined how various \edit{side} information about a learning problem can effectively help in generalization. 
We focused our attention on \edit{several} types of \edit{side} information, leading to linear, \edit{polygonal}, quadratic \edit{and conic} constraints, giving motivating examples and deriving complexity measure bounds. This work goes beyond the traditional paradigm of ball-like hypothesis spaces to study more exotic, yet realistic, hypothesis spaces, and is a starting point for more work on other interesting hypothesis spaces.

\editnew{\section*{\AppendixA: Quantifying the impact of side knowledge}}

\editnew{Here we describe an experiment\footnote{The source code is available at \url{https://github.com/thejat/supervised_learning_with_side_knowledge}.} that we did to demonstrate the impact of side knowledge encoded as polygonal (which subsumes linear), quadratic and conic constraints. Our goal was to compare predictive accuracies of a model that used side knowledge to a baseline model that did not use side knowledge.}

\editnew{\textit{Algorithm setups and performance measure:} We measured performance in terms of RMSE (Root Mean Squared Error) for models obtained from five setups: (1) multiple linear regression, (2) ridge regression, (3) ridge regression with polygonal constraints, (4) ridge regression with convex quadratic constraints, and (5) ridge regression with multiple conic constraints.}

\editnew{\textit{Dataset:} The dataset for this problem was generated using a multidimensional Gaussian distribution (with a fixed covariance matrix). The number of features was set to 60. A coefficient vector was arbitrarily chosen and the response variable was computed as a linear function of the coefficient vector and the feature vector with some additional Gaussian noise. Three types of samples (feature-label pairs) were generated: (a) A test sample of size 750 was kept aside during learning. The prediction performance numbers reported in Figure \ref{fig:experiment} were computed on this sample. (b) A ``knowledge sample" of size 120 was generated in order to incorporate side knowledge as polygonal, quadratic and conic constraints. For all three types of side knowledge, the same ``knowledge sample" was used, but different side knowledge was derived from it for the different algorithm setups. For polygonal (or multiple linear) constraints, a poset constraint (see Section \ref{subsec:linear}) of the form $\beta^{T}(\xt_{i} - \xt_{j}) \leq \yt_i - \yt_j$ was constructed for each pair of points in the knowledge set and a subset were chosen for use in the convex formulation (1200 linear constraints out of a possible 7140). A quadratic constraint of the form $\|\Gamma\XunlabT\beta\|_{2}^{2} \leq c$ was constructed to impose a smoothness side knowledge (see Section \ref{subsec:quadratic}). For this, the examples in the knowledge set were first sorted according to $\yt_i$ to be monotonic and the rows of $\XunlabT$ were reordered accordingly before being used in the constraint. The right hand side parameter $c$ of the quadratic constraint was defined to be $\sum_{i=1}^{119}(\yt_i - \yt_{i+1})^2$ and $\Gamma$ was a $119\times 120$ matrix with $\Gamma_{i,i} = 1$ and $\Gamma_{i,i+1} = -1$ for $i=1,..,119$.   One conic constraint for each example in the knowledge set was generated of the form $\beta^T\xt_i + r\|\beta\|_2 \leq \yt_i + r\|\beta^*\|_2$ (see Section \ref{subsec:conic}). Here, the parameter $r$ was a fixed positive real number and $\beta^*$ is the true underlying coefficient vector. Knowledge of the true underlying coefficient vector is not necessary to impose such conic constraints in practice (and was used here for ease of simulation only). (c) Thirty separate training samples of size 750 were generated. Thus, each time a model was trained, it was trained on one of 30 training sets, using constraints derived from the ``knowledge sample" (if it was an algorithm setup that used side knowledge) and tested on the test set.}

\editnew{\textit{Experimental Setup:} For each training sample (there are 30 of them), and for each of the 5 setups, we constructed a model by solving a convex program. (For the ridge regression methods, we also performed 5-fold cross validation to choose the hyper-parameter corresponding to the $\ell_2$-norm regularization term.) We then evaluated each model on the test sample and computed the RMSE. Further, to show dependence on training set size, for each training sample, we changed the data that we used from 300 examples to the full 750 examples (4 training set sizes - 300, 450, 600, 750). In summary, we learned (5 algorithm setups)*(4 training set sizes)*(30 training sets) = 600 models in this experiment, not including cross validation. Figure \ref{fig:experiment} shows the median RMSE (with 25th and 75th quantiles as whiskers) that we obtain across the 30 models. }

\editnew{\textit{Results:} We expected a performance increase over standard multiple linear regression when we impose polygonal, quadratic and conic constraints. As seen from Figure \ref{fig:experiment}, this is indeed true. Most prominently, the distribution of RMSE error values shifts downwards when side knowledge is used. As the sample size increases, the difference in performance between a ridge regression model learned without side knowledge and those learned with side knowledge decreases as expected; the side knowledge becomes less useful when more data are available to learn from.}

\begin{figure}
\centering
	\includegraphics[width=0.7\textwidth]{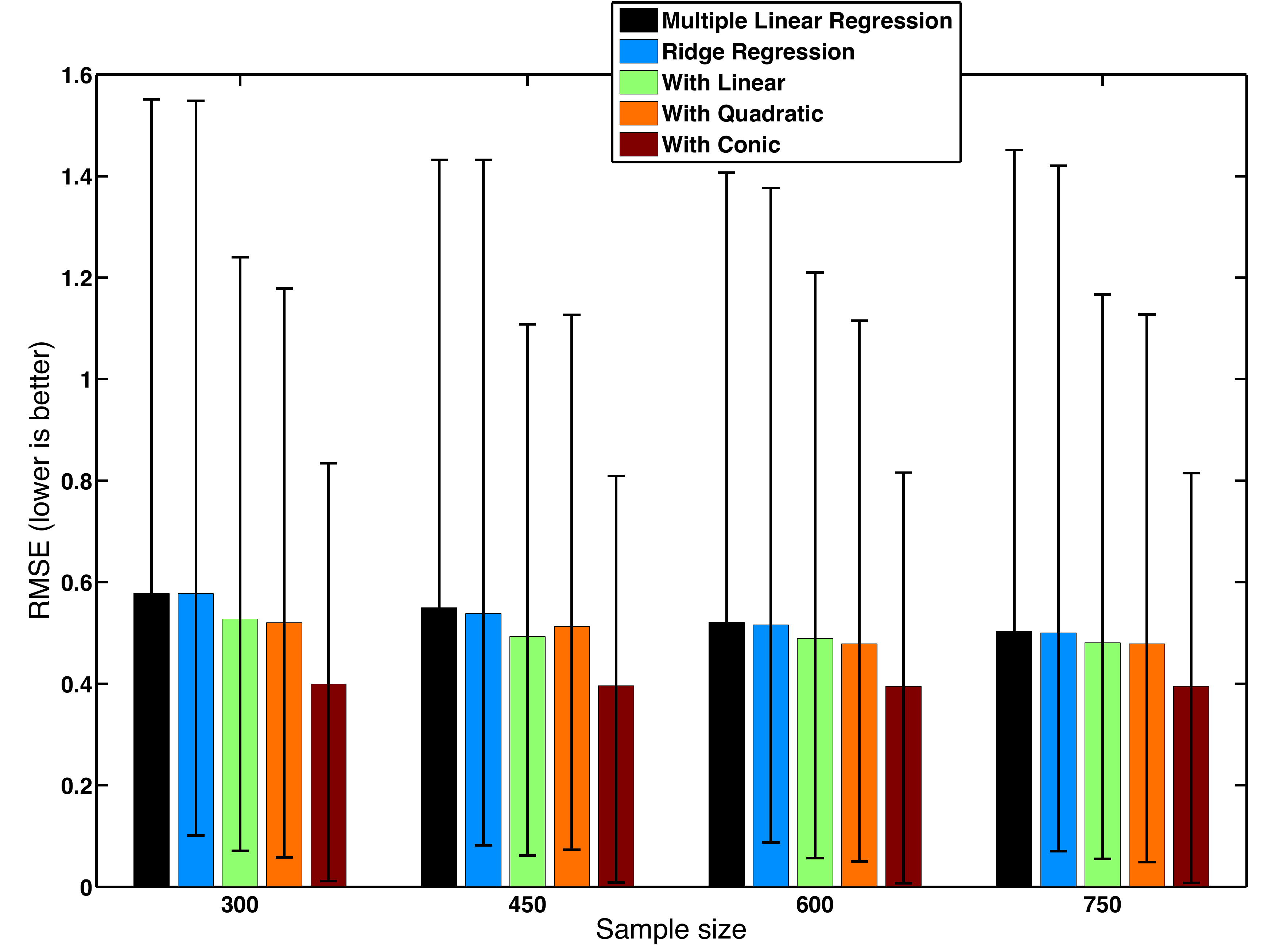}
\caption{\editnew{Plot of predictive performance (RMSE) of models learned using different learning formulations (see the legend). Models learned using side knowledge outperform the baseline multiple linear regression and ridge regression models as evidenced by the downward shift in the 25th-75th quantile ranges (shown as whiskers on the median bar-plots). For each sample size (300, 450, 600, 750), 30 training samples were generated and used to learn 30 different models in each modeling setup (with and without the various forms of side knowledge).} \label{fig:experiment}}
\end{figure}


\bibliographystyle{plainnat}
\bibliography{bib_theja_structure}

\begin{thebibliography}{39}
\providecommand{\natexlab}[1]{#1}
\providecommand{\url}[1]{\texttt{#1}}
\expandafter\ifx\csname urlstyle\endcsname\relax
  \providecommand{\doi}[1]{doi: #1}\else
  \providecommand{\doi}{doi: \begingroup \urlstyle{rm}\Url}\fi

\bibitem[Balcan and Blum(2005)]{balcan2005pac}
M.F. Balcan and A.~Blum.
\newblock A {PAC}-style model for learning from labeled and unlabeled data.
\newblock In \emph{Proceedings of Conference on Learning Theory}, pages 69--77.
  Springer, 2005.

\bibitem[Bartlett and Mendelson(2002)]{bartlett02}
Peter~L. Bartlett and Shahar Mendelson.
\newblock {Gaussian and Rademacher complexities: Risk bounds and structural
  results}.
\newblock \emph{Journal of Machine Learning Research}, 3:\penalty0 463--482,
  2002.

\bibitem[Basu et~al.(2006)Basu, Bilenko, Banerjee, and
  Mooney]{basu2006probabilistic}
Sugato Basu, Mikhail Bilenko, Arindam Banerjee, and Raymond~J Mooney.
\newblock Probabilistic semi-supervised clustering with constraints.
\newblock In \emph{Semi-supervised learning}, pages 71--98. Cambridge, MA. MIT
  Press, 2006.

\bibitem[Belkin and Niyogi(2004)]{belkin2004semi}
M.~Belkin and P.~Niyogi.
\newblock Semi-supervised learning on riemannian manifolds.
\newblock \emph{Machine Learning}, 56\penalty0 (1):\penalty0 209--239, 2004.

\bibitem[Belkin et~al.(2004)Belkin, Matveeva, and
  Niyogi]{belkin2004regularization}
Mikhail Belkin, Irina Matveeva, and Partha Niyogi.
\newblock Regularization and semi-supervised learning on large graphs.
\newblock In \emph{Proceedings of Conference on Learning Theory}, pages
  624--638. Springer, 2004.

\bibitem[Boyd and Vandenberghe(2004)]{boyd2004convex}
Stephen~P Boyd and Lieven Vandenberghe.
\newblock \emph{Convex optimization}.
\newblock Cambridge university press, 2004.

\bibitem[Chandrasekaran et~al.(2012)Chandrasekaran, Recht, Parrilo, and
  Willsky]{chandrasekaran2012convex}
Venkat Chandrasekaran, Benjamin Recht, Pablo~A Parrilo, and Alan~S Willsky.
\newblock The convex geometry of linear inverse problems.
\newblock \emph{Foundations of Computational Mathematics}, 12\penalty0
  (6):\penalty0 805--849, 2012.

\bibitem[Chang et~al.(2008{\natexlab{a}})Chang, Ratinov, and
  Roth]{chang2008constraints}
M~Chang, Lev Ratinov, and Dan Roth.
\newblock Constraints as prior knowledge.
\newblock In \emph{ICML Workshop on Prior Knowledge for Text and Language
  Processing}, pages 32--39, 2008{\natexlab{a}}.

\bibitem[Chang et~al.(2008{\natexlab{b}})Chang, Ratinov, Rizzolo, and
  Roth]{chang2008learning}
Ming-Wei Chang, Lev-Arie Ratinov, Nicholas Rizzolo, and Dan Roth.
\newblock Learning and inference with constraints.
\newblock In \emph{AAAI Conference on Artificial Intelligence}, pages
  1513--1518, 2008{\natexlab{b}}.

\bibitem[Fung et~al.(2002)Fung, Mangasarian, and Shavlik]{fung2002knowledge}
Glenn~M Fung, Olvi~L Mangasarian, and Jude~W Shavlik.
\newblock Knowledge-based support vector machine classifiers.
\newblock In \emph{Proceedings of Neural Information Processing Systems}, pages
  521--528, 2002.

\bibitem[G{\'o}mez-Chova et~al.(2008)G{\'o}mez-Chova, Camps-Valls, Munoz-Mari,
  and Calpe]{gomez2008semisupervised}
Luis G{\'o}mez-Chova, Gustavo Camps-Valls, Jordi Munoz-Mari, and Javier Calpe.
\newblock Semisupervised image classification with laplacian support vector
  machines.
\newblock \emph{Geoscience and Remote Sensing Letters, IEEE}, 5\penalty0
  (3):\penalty0 336--340, 2008.

\bibitem[James et~al.(2014)James, Paulson, and Rusmevichientong]{classo}
G.~M James, C~Paulson, and P~Rusmevichientong.
\newblock The constrained lasso.
\newblock \emph{working paper}, 2014.

\bibitem[John(1948)]{fritzjohn48}
Fritz John.
\newblock Extremum problems with inequalities as subsidiary conditions.
\newblock \emph{Studies and Essays Presented to R. Courant on his 60th
  Birthday, January 8, 1948}, pages 187--204, 1948.

\bibitem[K{\"a}{\"a}ri{\"a}inen(2005)]{kaariainen2005generalization}
Matti K{\"a}{\"a}ri{\"a}inen.
\newblock Generalization error bounds using unlabeled data.
\newblock In \emph{Proceedings of Conference on Learning Theory}, pages
  127--142. Springer, 2005.

\bibitem[Kahan(1968)]{kahan1968circumscribing}
W.~Kahan.
\newblock Circumscribing an ellipsoid about the intersection of two ellipsoids.
\newblock \emph{Canadian Mathematical Bulletin}, 11\penalty0 (3):\penalty0
  437--441, 1968.

\bibitem[Kakade et~al.(2008)Kakade, Sridharan, and
  Tewari]{kakade2008complexity}
S.M. Kakade, K.~Sridharan, and A.~Tewari.
\newblock On the complexity of linear prediction: Risk bounds, margin bounds,
  and regularization.
\newblock \emph{Proceedings of Neural Information Processing Systems}, 22,
  2008.

\bibitem[Kolmogorov and Tikhomirov(1959)]{kol61}
Andrey~Nikolaevich Kolmogorov and Vladimir~Mikhailovich Tikhomirov.
\newblock {$\varepsilon$-entropy and $\varepsilon$-capacity of sets in function
  spaces}.
\newblock \emph{Uspekhi Matematicheskikh Nauk}, 14\penalty0 (2):\penalty0
  3--86, 1959.

\bibitem[Lanckriet et~al.(2003)Lanckriet, Ghaoui, Bhattacharyya, and
  Jordan]{lanckriet2003robust}
Gert~RG Lanckriet, Laurent~El Ghaoui, Chiranjib Bhattacharyya, and Michael~I
  Jordan.
\newblock A robust minimax approach to classification.
\newblock \emph{The Journal of Machine Learning Research}, 3:\penalty0
  555--582, 2003.

\bibitem[Lauer and Bloch(2008)]{lauer2008incorporating}
Fabien Lauer and G{\'e}rard Bloch.
\newblock Incorporating prior knowledge in support vector machines for
  classification: A review.
\newblock \emph{Neurocomputing}, 71\penalty0 (7):\penalty0 1578--1594, 2008.

\bibitem[Le et~al.(2006)Le, Smola, and G{\"a}rtner]{le2006simpler}
Quoc~V Le, Alex~J Smola, and Thomas G{\"a}rtner.
\newblock Simpler knowledge-based support vector machines.
\newblock In \emph{Proceedings of the 23rd international conference on Machine
  learning}, pages 521--528. ACM, 2006.

\bibitem[Lobo et~al.(1998)Lobo, Vandenberghe, Boyd, and
  Lebret]{lobo1998applications}
Miguel~Sousa Lobo, Lieven Vandenberghe, Stephen Boyd, and Herv{\'e} Lebret.
\newblock Applications of second-order cone programming.
\newblock \emph{Linear algebra and its applications}, 284\penalty0
  (1):\penalty0 193--228, 1998.

\bibitem[Lu and Leen(2004)]{lu2004semi}
Zhengdong Lu and Todd~K Leen.
\newblock Semi-supervised learning with penalized probabilistic clustering.
\newblock In \emph{Proceedings of Neural Information Processing Systems}, pages
  849--856, 2004.

\bibitem[Maurer(2006)]{maurer2006rademacher}
Andreas Maurer.
\newblock {The {Rademacher} complexity of linear transformation classes}.
\newblock In \emph{Proceedings of Conference on Learning Theory}, pages 65--78.
  Springer, 2006.

\bibitem[Nguyen and Caruana(2008{\natexlab{a}})]{nguyen2008classification}
Nam Nguyen and Rich Caruana.
\newblock Classification with partial labels.
\newblock In \emph{Proceedings of the 14th ACM SIGKDD International Conference
  on Knowledge Discovery and Data Mining}, pages 551--559. ACM,
  2008{\natexlab{a}}.

\bibitem[Nguyen and Caruana(2008{\natexlab{b}})]{nguyen2008improving}
Nam Nguyen and Rich Caruana.
\newblock Improving classification with pairwise constraints: a margin-based
  approach.
\newblock In \emph{Machine Learning and Knowledge Discovery in Databases},
  pages 113--124. Springer, 2008{\natexlab{b}}.

\bibitem[Rigollet(2007)]{rigollet2006generalization}
Philippe Rigollet.
\newblock Generalization error bounds in semi-supervised classification under
  the cluster assumption.
\newblock \emph{Journal of Machine Learning Research}, 8:\penalty0 1369--1392,
  2007.

\bibitem[Shental et~al.(2004)Shental, Bar-Hillel, Hertz, and
  Weinshall]{shental2003computing}
Noam Shental, Aharon Bar-Hillel, Tomer Hertz, and Daphna Weinshall.
\newblock {Computing Gaussian mixture models with EM using equivalence
  constraints}.
\newblock In \emph{Proceedings of Neural Information Processing Systems},
  volume~16, pages 465--472, 2004.

\bibitem[Shivaswamy et~al.(2006)Shivaswamy, Bhattacharyya, and
  Smola]{shivaswamy2006second}
Pannagadatta~K Shivaswamy, Chiranjib Bhattacharyya, and Alexander~J Smola.
\newblock Second order cone programming approaches for handling missing and
  uncertain data.
\newblock \emph{The Journal of Machine Learning Research}, 7:\penalty0
  1283--1314, 2006.

\bibitem[Singh et~al.(2008)Singh, Nowak, and Zhu]{singh2008unlabeled}
Aarti Singh, Robert Nowak, and Xiaojin Zhu.
\newblock Unlabeled data: Now it helps, now it doesn't.
\newblock In \emph{Proceedings of Neural Information Processing Systems}, pages
  1513--1520, 2008.

\bibitem[Stojnic(2009)]{stojnic2009various}
Mihailo Stojnic.
\newblock Various thresholds for l1-optimization in compressed sensing.
\newblock \emph{arXiv preprint arXiv:0907.3666}, 2009.

\bibitem[Talagrand(2005)]{talagrand2005generic}
M.~Talagrand.
\newblock \emph{The generic chaining}.
\newblock Springer, 2005.

\bibitem[Towell et~al.(1990)Towell, Shavlik, and
  Noordewier]{towell1990refinement}
Geofrey~G Towell, Jude~W Shavlik, and M~Noordewier.
\newblock Refinement of approximate domain theories by knowledge-based neural
  networks.
\newblock In \emph{Proceedings of the Eighth National Conference on Artificial
  Intelligence}, pages 861--866. Boston, MA, 1990.

\bibitem[Tsirelson et~al.(1976)Tsirelson, Ibragimov, and
  Sudakov]{tsirel1976norms}
B.~S. Tsirelson, I.~A. Ibragimov, and V.~N. Sudakov.
\newblock Norms of gaussian sample functions.
\newblock In \emph{Proceedings of the Third Japan--U.S.S.R. Symposium on
  Probability Theory. Lecture Notes in Math.}, volume 550, pages 20--41.
  Springer, 1976.

\bibitem[Tulabandhula and Rudin(2013)]{TulabandhulaRu13}
Theja Tulabandhula and Cynthia Rudin.
\newblock Machine learning with operational costs.
\newblock \emph{Journal of Machine Learning Research}, 14:\penalty0 1989--2028,
  2013.

\bibitem[Tulabandhula and Rudin(2014)]{TuRu14progress}
Theja Tulabandhula and Cynthia Rudin.
\newblock On combining machine learning with decision making.
\newblock \emph{Machine Learning}, 97\penalty0 (1-2):\penalty0 33--64, 2014.

\bibitem[Vapnik(1998)]{vapnik98}
Vladimir~Naumovich Vapnik.
\newblock \emph{{Statistical learning theory}}, volume~2.
\newblock Wiley New York, 1998.

\bibitem[Wainwright(2011)]{wainwrightNotes}
Martin Wainwright.
\newblock \emph{Metric entropy and its uses (Chapter 3)}.
\newblock Unpublished draft, 2011.

\bibitem[Zhang(2002)]{zhang02}
Tong Zhang.
\newblock {Covering number bounds of certain regularized linear function
  classes}.
\newblock \emph{Journal of Machine Learning Research}, 2:\penalty0 527--550,
  2002.

\bibitem[Zhu(2005)]{zhu05survey}
Xiaojin Zhu.
\newblock Semi-supervised learning literature survey.
\newblock Technical Report 1530, Computer Sciences, University of
  Wisconsin-Madison, 2005.

\end{thebibliography}
\end{document}